\documentclass[10pt,twocolumn,letterpaper]{article}

\usepackage[pagenumbers]{wacv} 

\definecolor{wacvblue}{rgb}{0.21,0.49,0.74}
\usepackage[pagebackref,breaklinks,colorlinks,allcolors=wacvblue]{hyperref}

\def\wacvPaperID{1074} 
\def\confName{WACV}
\def\confYear{2026}

\usepackage{tikz}

\usepackage{multirow} 
\usepackage{algorithm}
\usepackage[noend]{algpseudocode}
\algrenewcommand\algorithmicrequire{\textbf{Input:}}
\algrenewcommand\algorithmicensure{\textbf{Output:}}
\usepackage[capitalize]{cleveref}
\Crefname{equation}{Eq.}{Eqs.}
\Crefname{appendix}{App.}{Apps.}
\Crefname{figure}{Fig.}{Figs.}
\Crefname{tabular}{Tab.}{Tabs.}
\Crefname{algorithm}{Alg.}{Algs.}
\usepackage{amsmath}

\DeclareMathOperator*{\argmin}{arg\,min}
\newcommand{\bm}{\boldsymbol}
\newcommand{\w}{\bm \omega}
\newcommand{\D}{\mathcal{D}}
\newcommand{\M}{\mathcal{M}}
\newcommand{\Se}{\mathcal{S}}
\newcommand{\N}{\mathcal{N}}
\newcommand{\A}{\mathcal{A}}
\newcommand{\ba}{\boldsymbol a}
\newcommand{\bc}{\boldsymbol c}
\newcommand{\btheta}{\boldsymbol \theta}

\usepackage{xcolor}
\definecolor{myblue}{RGB}{173, 216, 230}
\usepackage{enumitem}
\usepackage{mdframed}
\usepackage{amsthm}
\usepackage{colortbl}
\usepackage{tcolorbox}
\newcommand{\tikzxmark}{%
\tikz[scale=0.23] {
    \draw[line width=0.7,line cap=round] (0,0) to [bend left=6] (1,1);
    \draw[line width=0.7,line cap=round] (0.2,0.95) to [bend right=3] (0.8,0.05);
}}
\newcommand{\tikzcmark}{%
\tikz[scale=0.23] {
    \draw[line width=0.7,line cap=round] (0.25,0) to [bend left=10] (1,1);
    \draw[line width=0.8,line cap=round] (0,0.35) to [bend right=1] (0.23,0);
}}

\newcommand{\bftab}{\fontseries{b}\selectfont}
\newcommand{\maE}{\mathbb E}

\newtheorem{theorem}{Theorem}
\newtheorem{lemma}{Lemma}

\newtheorem{assumption}{Assumption}
\newtheorem{remark}{Remark}

\title{Guided Model Merging for Hybrid Data Learning: Leveraging Centralized Data to Refine Decentralized Models}

\author{
Junyi Zhu$^{1}$\thanks{Corresponding to \{junyi.zhu, savas.ozkan, m.ozay\}@samsung.com} \quad Ruicong Yao$^{2}$ \quad  Taha Ceritli$^{1}$ \quad Savas Ozkan$^{1}$\footnotemark[1] \quad  Matthew B. Blaschko$^{2}$ 
\\ 
\quad Eunchung Noh$^{3}$  \quad Jeongwon Min$^{3}$ \quad Cho Jung Min$^{3}$ \quad Mete Ozay$^{1}$\footnotemark[1] \\
$^1$Samsung R\&D Institute UK (SRUK) \quad
$^2$KU Leuven, Belgium \quad
$^3$Samsung Electronics Korea \\
}

\begin{document}
\maketitle
\begin{abstract}
Current network training paradigms primarily focus on either centralized or decentralized data regimes. However, in practice, data availability often exhibits a hybrid nature, where both regimes coexist. This hybrid setting presents new opportunities for model training, as the two regimes offer complementary trade-offs: decentralized data is abundant but subject to heterogeneity and communication constraints, while centralized data—though limited in volume and potentially unrepresentative—enables better curation and high-throughput access. Despite its potential, effectively combining these paradigms remains challenging, and few frameworks are tailored to hybrid data regimes. To address this, we propose a novel framework that constructs a model atlas from decentralized models and leverages centralized data to refine a global model within this structured space. The refined model is then used to reinitialize the decentralized models. Our method synergizes federated learning (to exploit decentralized data) and model merging (to utilize centralized data), enabling effective training under hybrid data availability. Theoretically, we show that our approach achieves faster convergence than methods relying solely on decentralized data, due to variance reduction in the merging process. Extensive experiments demonstrate that our framework consistently outperforms purely centralized, purely decentralized, and existing hybrid-adaptable methods. Notably, our method remains robust even when the centralized and decentralized data domains differ or when decentralized data contains noise, significantly broadening its applicability.
\end{abstract}
\vspace{-1.5em}
\section{Introduction}
\label{sec:intro}
\vspace{-.3em}

\begin{figure}[t]
\centering
\includegraphics[width=0.90\columnwidth]{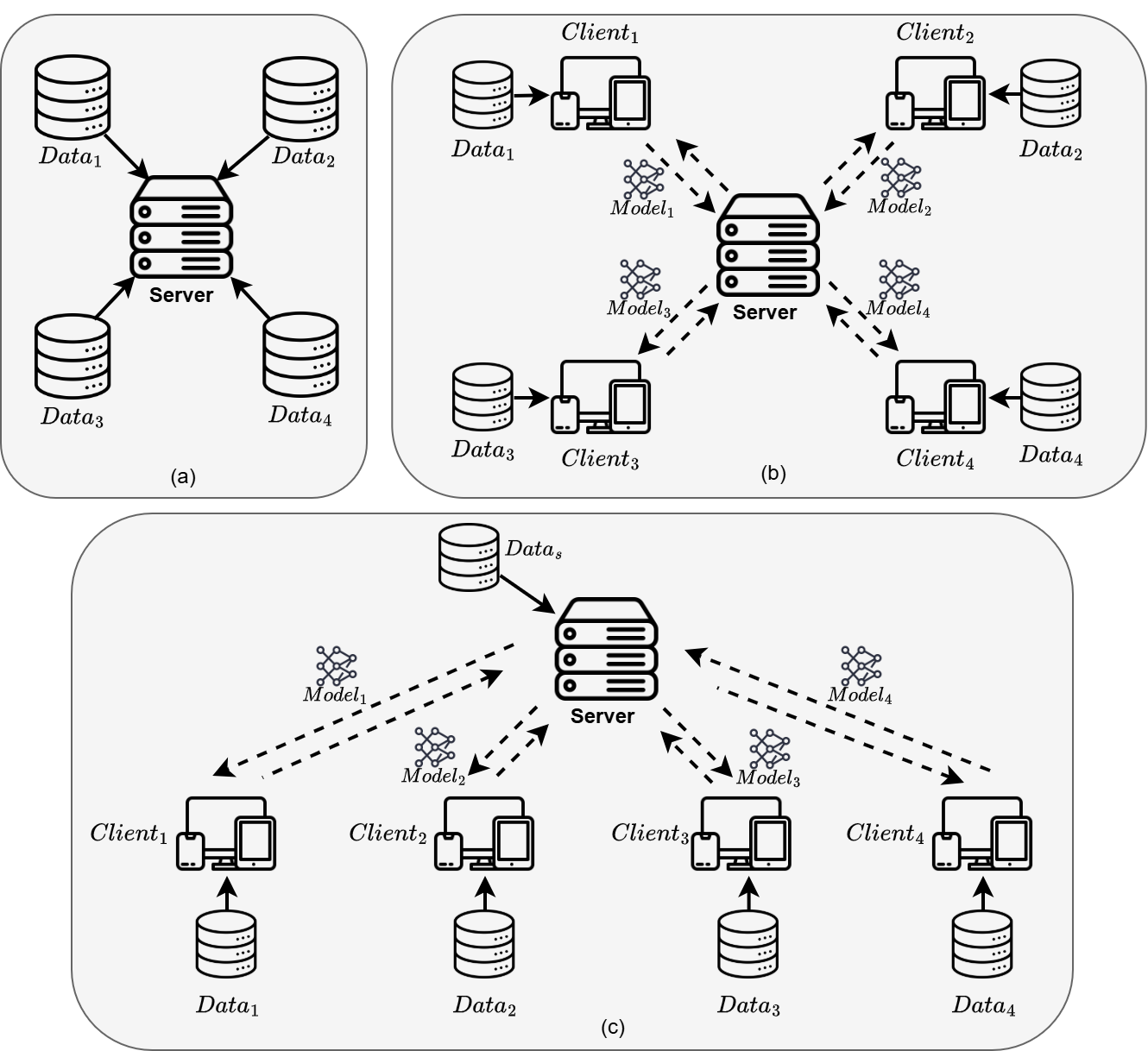}
\vspace{-1em}
\caption{Illustration of data regimes.
\textbf{(a)} Centralized regime: all data is aggregated at the server for training.
\textbf{(b)} Decentralized regime: data remains distributed across clients, which train local models and share updates with the server.
\textbf{(c)} Hybrid regime: decentralized learning is performed while a centralized dataset is concurrently available to assist the training process.
We follow standard FL terminology, referring to the central node as the \textit{server} and the distributed nodes as \textit{clients}.}

\label{fig:data_regime}
\end{figure}

Modern network training has been significantly advanced by centralized learning paradigm where data is aggregated in one location (see \cref{fig:data_regime}-a )~\citep{brown2020language,radford2021learning,rombach2022high,kirillov2023segment}. While centralized data offers possibility for data filtering and high-throughput access (using short-range high speed connection), creating comprehensive centralized datasets is often costly or even prohibitive. For example, using crowd-worker to make large-scale dataset could be expensive, while aggregating data from individual users is complicated by device limitations and user consent, or from medical centers by regulatory constraints. 
To overcome these limitations, learning from decentralized data—while retaining it on local devices—has emerged as a prominent research direction. A widely adopted approach is Federated Learning (FL), where clients (i.e., decentralized nodes) train models locally and transmit their updates to a central server for aggregation (see \cref{fig:data_regime}-b)\citep{fedavg}. However, decentralized data is inherently siloed and typically exhibits significant heterogeneity\citep{kairouz2021advances,scaffold,li2021model,gao2022feddc,wang_obj_inconsistency}. Moreover, clients in form of edge devices can be unresponsive or introduce delays in communication during training~\citep{xie2019asynchronous,nguyen2022federated,wang2023tackling}. These practical challenges often degrade the performance of decentralized learning relative to centralized training.

\vspace{+.3em}
\noindent \textbf{Hybrid Data Regime.} Rather than strictly adhering to fully centralized or fully decentralized approaches, we focus on a third, practically prevalent setting: the hybrid data regime (see \cref{fig:data_regime}-c). In this setting, the server possesses some data, while a substantial portion remains decentralized. This reflects a natural distribution of data across centralized and decentralized sources. For example, in 2024, an estimated 1.94 trillion photos were taken worldwide,\footnote{Source:  https://photutorial.com/photos-statistics/} yet the largest public image dataset—such as LAION-5B~\citep{schuhmann2022laion}—contains only several billion images.
In this paper, we explore how to leverage hybrid data effectively and highlight two key findings:
\begin{tcolorbox}[colback=gray!5, colframe=gray!80, boxrule=0.5pt, arc=2pt, left=5pt, right=5pt, top=5pt, bottom=5pt]
\textit{Finding 1}: A small amount of centralized data can guide a large quantity of scattered decentralized data to outperform methods that rely solely on either centralized or decentralized data—even when the decentralized data is noisy.

\textit{Finding 2}: Centralized data can effectively guide the learning of decentralized models even when the two come from different domains.
\end{tcolorbox}


\textcolor{wacvblue}{\textit{We emphasize that our work does not aim to encourage the explicit collection of hybrid data. Rather, we observe that the hybrid data regime naturally arises in practice, and our goal is to exploit such existing data availability.}}
Below, we categorize two types of data availability within the hybrid regime, along with representative practical scenarios:

\begin{itemize}[leftmargin=8pt,itemsep=0pt,topsep=0pt]
\item \textbf{In-Domain (ID) Data Availability}: The server holds data that aligns with the clients’ task domains. Practical sources include:
	\textbf{(1)}	\textit{Public datasets}: Existing public datasets that match the target task.
	\textbf{(2)}	\textit{Data curation}: The server operator (e.g., a company) may pay crowd workers to curate a task-specific dataset.
	\textbf{(3)}	\textit{Incentive mechanisms}: Clients may share a portion of their data with the server in exchange for incentives.
	\textbf{(4)}	\textit{Trusted Execution Environment (TEE)}: A TEE can be deployed on the server, allowing clients to securely transmit data that the server can manage but not access~\cite{tee}.

\item \textbf{Out-of-Domain (OOD) Data Availability}: The server holds data from domains different from those of the clients. In this case, public datasets—though unrelated to the clients’ tasks—can serve as centralized datas.
\end{itemize}

\vspace{+.5em}
\noindent \textbf{Challenges in Decentralized Learning.} In this work, we build upon the fundamental framework of federated learning (FL) to learn from decentralized data. \textcolor{wacvblue}{Our study considers the core challenges of \textit{data heterogeneity} encountered in FL, and the challenge of \textit{asynchronous communication}, where the server accepts delayed model updates.} Our communication setup is informed by stakeholder constraints, particularly targeting practical deployment on mobile platforms.
To address these challenges using hybrid data, we introduce \underline{Fed}erated \underline{D}ual \underline{Le}arning (\texttt{Feddle}), a framework that builds upon FL while \textit{enabling the server to optimize merging coefficients using either ID or OOD data}. \texttt{Feddle} allows the server to more accurately weigh client model updates—and, if beneficial, even assign negative weight. Theoretically, we show \texttt{Feddle} achieves faster convergence compared to existing methods. Empirically, it outperforms other hybrid approaches such as fine-tuning merged models on server-side data or training a separate server-side model and merging it with client models.

\noindent\textbf{Our contributions are summarized as follows}:
\begin{itemize}[leftmargin=10pt,itemsep=2pt,topsep=2pt]
\item \textbf{(1)} We formalize the concept of the hybrid data regime and demonstrate its potential to improve the utilization decentralized data.
\item \textbf{(2)} We introduce the model atlas to buffer communication fluctuations and define an efficient search space for server-side optimization.
\item \textbf{(3)} By leveraging a surrogate loss and a fallback mechanism, we show that OOD data can be effectively utilized at the server, thereby broadening the applicability of our framework.
\item \textbf{(4)} Extensive experiments demonstrate that \texttt{Feddle} consistently outperforms baseline methods under both ID and OOD data availability.
\item \textbf{(5)} We provide theoretical analysis showing that \texttt{Feddle} achieves a faster convergence rate compared to existing methods.
\end{itemize}

\vspace{-.5em}
\section{Related Work}
\label{sec:bg:rw}
\vspace{-.5em}
\noindent\textbf{Learning from Decentralized Data.}
Federated Learning (FL) has emerged as a prominent framework for learning from decentralized data~\citep{fedavg}. A core challenge in FL is \textit{data heterogeneity}\citep{kairouz2021advances,fedprox}, which leads to optimization difficulties and degraded convergence\citep{wang_obj_inconsistency,Li2020On}. To address this, various strategies have been proposed, including Bayesian modeling~\citep{chen2021fedbe,yurochkin2019bayesian,zhang_pfedbayes,Zhu_2023_CVPR}, variance reduction~\citep{scaffold}, contrastive learning on shared representations~\citep{li2021model}, and client clustering~\citep{ifca}.
Another line of research focuses on \textit{personalized FL (PFL)}, which trains client-specific models while regularizing them through shared information~\citep{zhang2021parameterized,song2022personalized,nikoloutsopoulos2022,dinh_pfedme}. However, PFL methods prioritize individual client performance and typically do not aim to learn a unified global model that generalizes across the full data distribution.
\textit{Communication delay} is another major obstacle in FL. Some works mitigate this by reducing the local training workload for slow clients~\citep{zhang2023timelyfl} or by discarding stale updates that exceed a delay threshold~\citep{liu2024fedasmu}. While effective to some extent, these strategies often sacrifice data coverage or model diversity. Tier-based architectures~\citep{chai2020tifl,9910131} and adaptive client sampling~\citep{chen2022optimal,qi2023fedsampling} offer alternative solutions, but introduce additional system complexity and may compromise training stability in practice.
In contrast, we adopt a line of work that embraces \textit{asynchronous communication}, allowing the server to incorporate delayed updates without strict synchronization. This leads to simpler and more robust communication protocols~\citep{xie2019asynchronous,nguyen2022federated,wang2023tackling,wang2024fadas,leconte2024queuing}, which align well with edge-device deployment scenarios.

\noindent\textbf{Learning from Hybrid Data.}
Some prior studies have explored scenarios in which the server computes model updates on behalf of clients that lack sufficient computational resources and choose to upload their data~\citep{elbir2021hybrid,ni2023semi,feng2023hybrid}. These approaches rely exclusively on ID data at the server, which limits their applicability.
Beyond that, several works leverage either ID or OOD server-side data to distill knowledge from clients into the global model~\citep{li2019fedmd,lin2020ensemble,yang2023fedfed}. However, these methods do not account for asynchronous communication—client updates are treated uniformly, irrespective of communication delays—making them less suited for real-world deployment.
Most closely related to our work, \citet{yueqioptimizing} also proposes optimizing merging coefficients. However, their method restricts coefficients to be strictly positive, whereas we show that allowing negative values can improve global model quality. Moreover, their approach assumes access to ID server data, while our framework can also leverage OOD data, enabling a wider range of application scenarios.

\vspace{-.5em}
\section{Problem Statement}
\vspace{-.5em}
\label{sec:bg}
In this section, we provide an overview of the background and key challenges associated with learning from decentralized data under our asynchronous communication setup.

\noindent\textbf{Federated Learning from Decentralized Data.}
FL assumes that the dataset $\D$ is fully partitioned across $J$ clients, such that $\D = \{\D_j\}_{j=1}^J$, where $j \in \{1,\ldots,J\}$ indexes the clients. A central server coordinates the clients to perform local training on their respective datasets and aggregates their model updates to form a global model. This process is repeated over $K$ communication rounds to optimize the global model toward the population distribution $p(\D)$.

At the beginning of each round $k$, the server broadcasts the current global model $\w^k$ to all clients. Each client $j$ then initializes its local model and performs local optimization:
\begin{equation}
\vspace{-.5em}
\label{eq:dfef}
\nonumber
\w_j = \w^k;\quad  \w_j^{k} = \argmin_{\w_j}\ell(\D_j, \w_j), \quad \forall j = 1, \ldots, J,
\end{equation}
where $\ell$ denotes the task-specific loss function. After local training, client $j$ computes its model update $\Delta \w_j^k = \w_j^k - \w^k$ and sends it to the server. The server then aggregates the updates ${\Delta \w_j^k}{j=1}^J$ using a predefined merging function $\M$ to obtain the next global model:
\begin{equation}
\label{eq:nvnv}
\w^{k+1} = \M(\Delta \w{1:J}^k, \w^k),
\end{equation}
e.g.\
$\M_{\text{\tiny FedAvg}}(\Delta \w_{1:J}^k, \w^k) = \w^k +
\sum_{j=1}^J\frac{|\D_j|}{|\D|} \Delta \w_j^{k}$~\citep{fedavg}.

\begin{figure}[t!]
    \centering
    \begin{subfigure}[t]{.48\columnwidth}
        \centering
        \includegraphics[width=\linewidth]{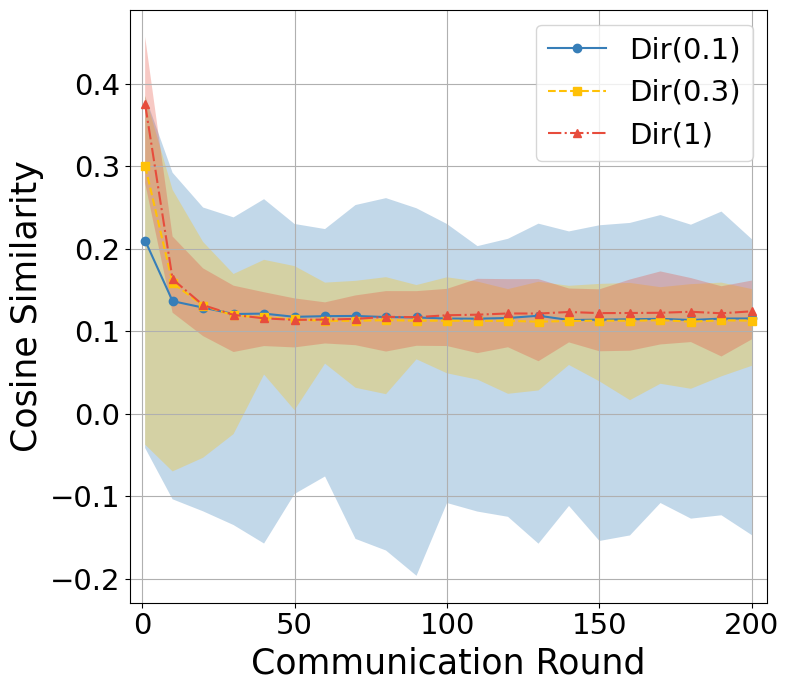}
        \caption{Cosine similarity between individual and true model updates.}
        \label{fig:mot:cosine}
    \end{subfigure}
    \hfill
    \begin{subfigure}[t]{.48\columnwidth}
        \centering
        \includegraphics[width=\linewidth]{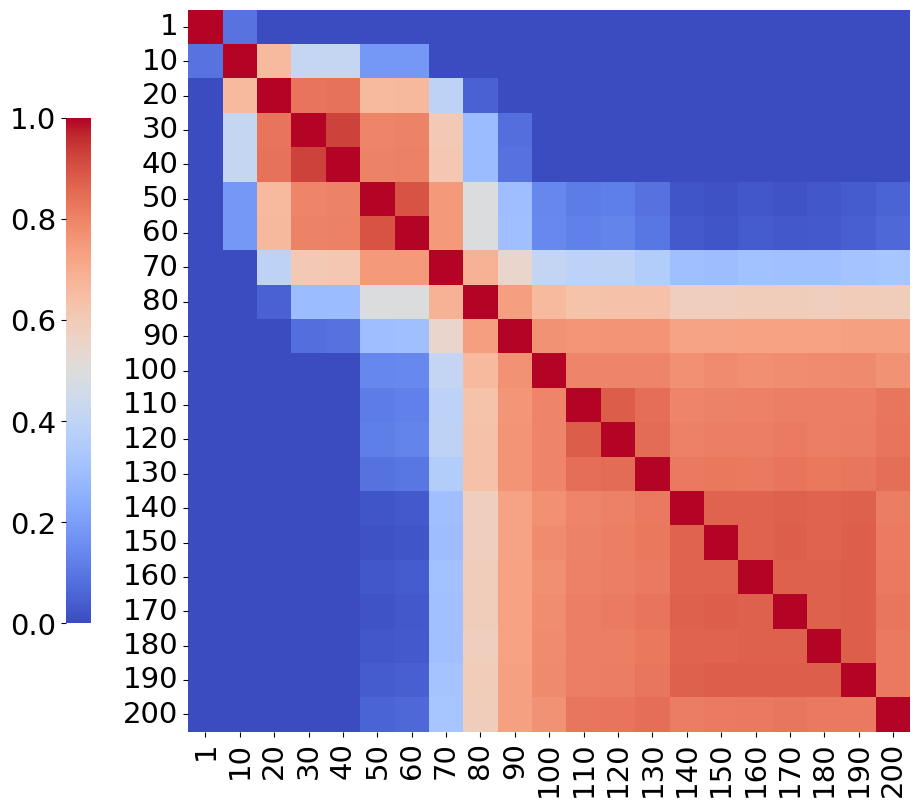}
        \caption{Cosine similarity between true model updates at different rounds; data is distributed w.r.t. Dir(0.1).}
        \label{fig:mot:corr1}
    \end{subfigure}%
\caption{Statistics of model updates in FL under varying degrees of data heterogeneity simulated using Dirichlet distribution (denoted as Dir$(\cdot)$) following previous work~\cite{yurochkin2019bayesian}. Subplot (a) displays mean values, with bands representing max.\ and min.\ values.}
\label{fig:motivation}
\end{figure}

\noindent\textbf{Heterogeneous Data Distribution.}
In practice, clients are often geographically distributed or operate in diverse environments, leading to non-identically distributed (non-IID) local datasets~\cite{kairouz2021advances}. As a result, even when clients share the same model initialization or prior knowledge, optimizing the data likelihood $p(\D_i \mid \w_i)$ for client $i$ and $p(\D_j \mid \w_j)$ for client $j$ leads to distinct posterior distributions $p(\w_i \mid \D_i)$ and $p(\w_j \mid \D_j)$. Consequently, the resulting model updates $\Delta \w_1, \ldots, \Delta \w_J$ also diverge~\cite{wang_obj_inconsistency,chen2021fedbe,zhang_pfedbayes,yurochkin2019bayesian,Zhu_2023_CVPR}.
As shown in \cref{fig:mot:cosine}, under strong heterogeneity—such as that induced by a Dirichlet distribution with concentration parameter $\alpha = 0.1$—clients may disagree on the optimization direction. In extreme cases, some client updates may even point opposite to the true global update, defined as $\Delta \w := \sum_j \frac{|\D_j|}{|\D|} \Delta \w_j$. This suggests that assigning \textit{uniformly positive} aggregation coefficients to all clients is suboptimal, despite its prevalence in existing FL methods.
\begin{figure*}[t]
\centering
\begin{minipage}[t]{0.30\textwidth}
    \centering
    \includegraphics[width=\textwidth]{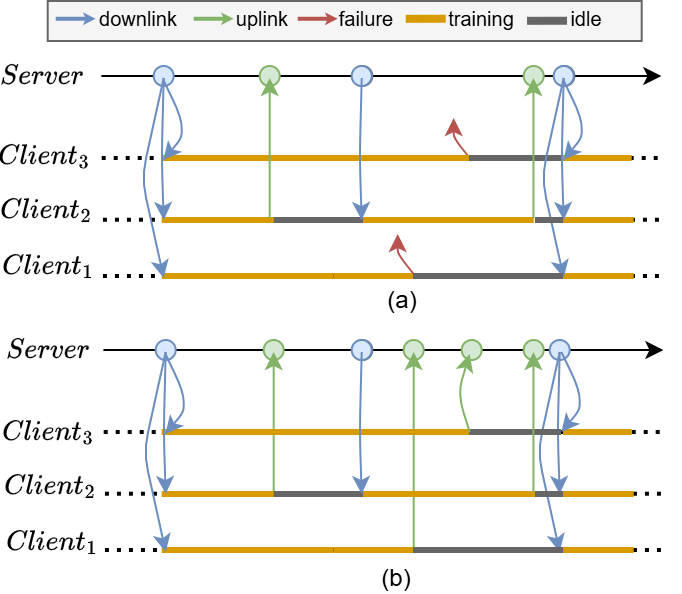}
    \caption{Illustration of synchronous (a) and asynchronous (b) communication in FL. Downlink is simplified for clarity.}
    \label{fig:async}
\end{minipage}
\hfill
\begin{minipage}[t]{0.64\textwidth}
    \centering
    \includegraphics[width=\textwidth]{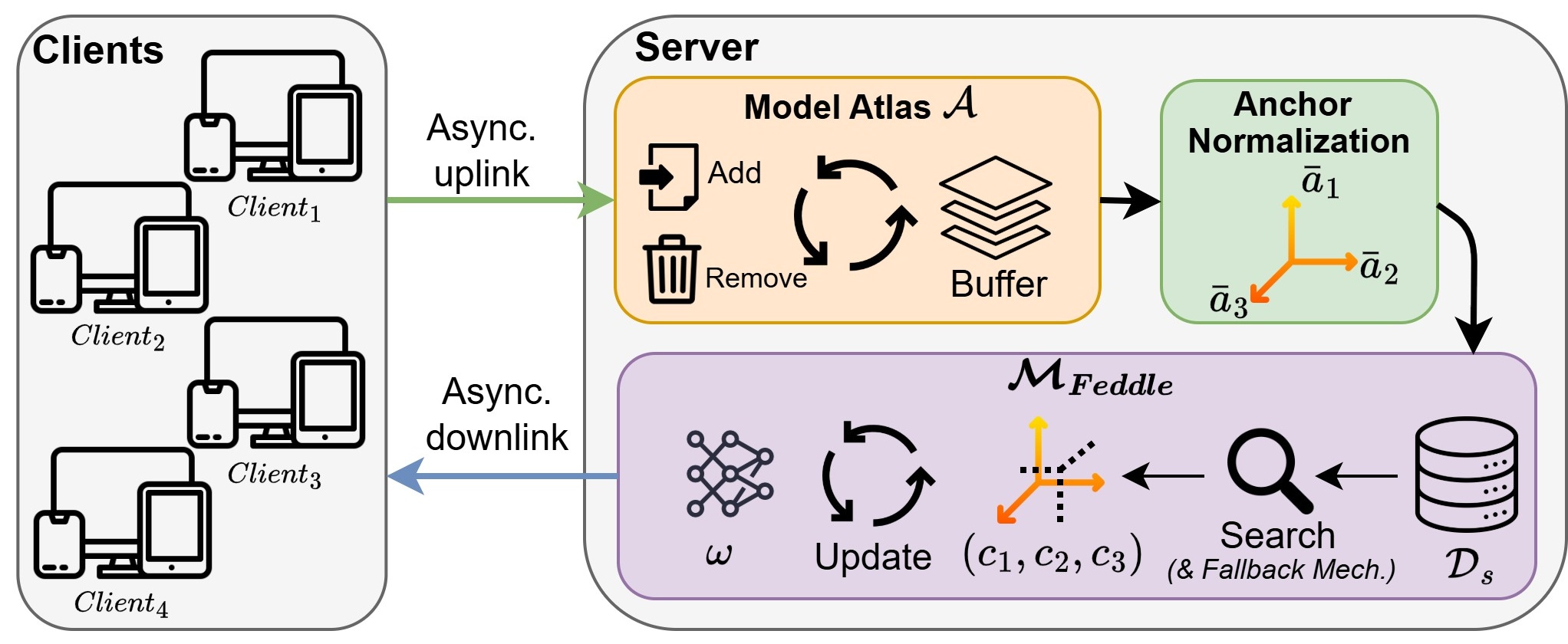}
    \caption{Overview of the \texttt{Feddle} framework. The server coordinates clients' local training using an asynchronous mechanism. Model atlas is updated by clients' model updates, which is then used to conduct coefficient search for the global model optimization.}
    \label{fig:arch}
\end{minipage}
\vspace{-1.em}
\end{figure*}

\noindent\textbf{Asynchronous Communication.}
\Cref{eq:nvnv,eq:dfef} implicitly assume that all clients coordinate and communicate with the server simultaneously—a requirement that is often impractical in real-world settings due to communication constraints. A common alternative~\citep{fedavg,bonawitz2019,wu2022communication,li2021model} is to adopt a \textit{synchronous} communication mechanism, where the server aggregates model updates after receiving results from a fixed number $N$ of clients, discarding any updates that arrive late (see~\cref{fig:async}):
\begin{equation}
\vspace{-.8em}
\label{eq:ffff}
\w^{k+1} = \M({\Delta\w_{ j_n}^{ k}}{n=1}^N, \w^k),
\end{equation}
where $\{j_n\}_{n=1}^N$ are the first $N$ clients to report in round $k$.
However, this approach suffers from several drawbacks:
\textbf{1)} The server must still wait for the slowest among the selected $N$ clients to respond;
\textbf{2)} many clients complete training but have their updates discarded, leading to wasted computation and energy;
\textbf{3)} some clients may consistently fail to report in time, biasing the global model toward a subset of the data distribution.
To mitigate these issues, a growing body of work~\citep{xie2019asynchronous,nguyen2022federated,wang2023tackling,wang2024fadas} adopts \textit{asynchronous communication}, where the server incorporates model updates as they arrive, regardless of delay (see \cref{fig:async}):
\begin{equation}
\vspace{-.8em}
\label{eq:aaaa}
\w^{k+1} = \M({\Delta\w_{ j_n}^{ k_n}}_{n=1}^N, \w^k),
\end{equation}
where $k_1, \ldots, k_N$ denote the downlink rounds when clients $j_1, \ldots, j_N$ received the global model (cf.\ \cref{eq:dfef}).
However, asynchronous communication introduces its own challenge: delayed updates may be stale, further compounding the optimization misalignment already caused by data heterogeneity (see \cref{fig:mot:corr1}).

\vspace{-.5em}
\section{Federated Dual Learning of Hybrid Data}
\vspace{-.5em}
\label{sec:method}

As discussed in \cref{sec:bg}, data heterogeneity leads to misaligned model updates, a problem further exacerbated by asynchronous communication. In the hybrid data setting, the server has access to a centralized dataset, which opens the door to data-guided aggregation strategies. Such strategies have shown promising results in the model merging literature, particularly when merging models fine-tuned from a shared pre-trained initialization~\citep{liu2024linear,akiba2025evolutionary,liu2024efficient}.
Motivated by this, we propose to leverage server-side data to optimize the merging coefficients used for aggregating client model updates. Unlike existing model merging approaches, however, our setting requires maintaining a buffer that caches asynchronously received model updates. This also necessitates mechanisms for identifying and removing low-quality updates in the buffer to prevent degradation of global model. Moreover, we demonstrate that even OOD server data can effectively guide the coefficient optimization process, making our method broadly applicable across diverse scenarios.

We introduce our proposed framework, \textit{Federated Dual Learning} (\texttt{Feddle}), in \cref{sec:method:arch}. A theoretical analysis of its convergence behavior is presented in \cref{sec:method:analysis}. \textit{An overview of the framework is illustrated in \cref{fig:arch}, and algorithmic details are provided in \cref{alg:server} (see \cref{app:alg}).}

\vspace{-.2em}
\subsection{Framework Architecture}
\label{sec:method:arch}
\vspace{-.2em}
In \cref{sec:method:atlas}, we describe the construction of model atlas $\A$, which consists of anchors $\{\ba_m\}_{m=1}^{|\A|}$ that identify the search space for the server to determine the corresponding merging coefficients $\{c_m\}_{m=1}^{|\A|}$. We then formulate the objective function for optimizing the merging coefficients under two distinct data availability scenarios in \cref{sec:method:obj}. In \cref{sec:method:fallback}, we introduce the fallback mechanism, which enhances the robustness of \texttt{Feddle} and has been observed to be crucial for successful learning with OOD data. We discuss computation efficiency in \cref{sec:computation}.

\vspace{-0.em}
\subsubsection{Model Atlas}
\vspace{-.5em}
\label{sec:method:atlas}
In \texttt{Feddle}, we introduce a model atlas $\A$ that defines the optimization space of the server. The atlas consists of maximum $M$ anchors $\A = \{\ba_m\}_{m=1}^{|\A|}, |\A| \leq M$, each representing an optimization direction. By utilizing client model updates as anchor points, the server can
optimize the global model in a subspace that has been explored by the clients, making it potentially more efficient. Notably, we find that setting $M$ to a relatively small value such as 20 is sufficient to effectively update a large global model with $10^7$
parameters. This flexibility in choosing $M$ allows the server to optimize the global model even when only limited data is available, mitigating the risk of overfitting.

\vspace{+0.3em}
\noindent\textbf{Addition and Removal of Anchors.}
When receiving a model update $\Delta \w_j$ from the $j$-th client, we add it as an anchor to the atlas. Initially, when the atlas is not full, we assign the next available index to the new anchor $\ba_{|\A|+1} := \Delta \w_j$. Once the atlas reaches its maximum size $M$, we remove an existing anchor to accommodate a new one. Instead of using a simple first-in-first-out (FIFO) strategy, we rank anchors based on their importance scores $\Se = \{s_m\}_{m=1}^M$ and remove the least important anchor $\ba_{m'}$, where $m' =\argmin_m \Se$. In \texttt{Feddle}, we use the absolute values $\{abs(c_m)\}_{m=1}^M$ of the aggregation coefficients $\{c_m\}_{m=1}^M$ found through the search as importance scores $s_m = abs(c_m), \forall m = 1,\ldots, M $, since they indicate how far the global model has moved in each direction.

\vspace{+0.3em}
\noindent\textbf{Anchor Normalization.}
The presence of data heterogeneity and delayed response from asynchronous communication, inevitably introduces variability in the magnitudes of the anchors (cf.\ \cref{fig:motivation}). This, in turn, can result in coefficients with disparate magnitudes, potentially leading to optimization challenges. To mitigate this issue, all anchors are normalized using the median of their $\ell_2$ norms before initiating the coefficient search at each round by:
\begin{equation}
    \label{eq:anc_normalization}
     \bar\ba_m =
\operatorname{median}(||\ba_1||,\ldots,||\ba_{|\A|}||)\cdot\frac{\ba_m}{||\ba_m||}, \quad \forall m.
\end{equation}
Note that we use $|\A|$ instead of $M$, since \texttt{Feddle} can perform coefficient search even before $|\A|$ reaches $M$.

\subsubsection{Search Objective}
\label{sec:method:obj}
\noindent\textbf{In-Domain Data Availability.}
When in-domain data $\D_S \sim \D$ is available, the server can perform a direct search using the loss function $\ell$ consistent with local training by
\begin{equation}
    \label{eq:kkkk}
    \{\hat c_m \}_{m=1}^{|\A|} = \argmin_{\boldsymbol c}\ell (\D_S, \w^k + \sum_{m=1}^{|\A|}c_m\bar\ba_m),
\end{equation}
where $\boldsymbol c=[c_1, \ldots, c_{|\A|}]$. In this work, we use the cross-entropy loss as our target task is multi-class classification. Once search is completed, the global model is updated by
\begin{equation}
\w^{k+1} = \w^k + \sum_{m=1}^{|\A|} \hat c_m\bar \ba_m.
\label{eq:updateWhat}
\end{equation}

\paragraph{Out-of-Domain Data Availability.}
When only OOD data $\D_S' \not\sim \D$ is available, we employ a surrogate loss function $h$ to shape the optimization landscape for the coefficient search. Ideally, $h(\D_S',\w)$ should exhibit a monotonic relation with $\ell(\D,\w)$, increasing and decreasing in tandem as $\w$ is updated by the coefficients $\bm c$. Namely, we require:

\vspace{-.5em}
\begin{equation}
    \label{eq:adfd}
     \langle\partial h(\D_S',\, \w)/\partial \boldsymbol c,\; \partial \ell(\D,\,\w)/\partial \boldsymbol c\rangle > 0.
\end{equation}

\noindent{\Cref{eq:adfd}} implies that the optimization direction of $h(\D_S',\w)$ aligns with that of $\ell(\D,\w)$. Interestingly, finding $h$ is \textbf{not} particularly challenging. For instance, if $\w$ performs well on $\D$, its feature extraction and representation capabilities should generalize to $\D_S'$.  Therefore, we can assess the quality of $\w$ by evaluating how well its representations of $\D_S’$ perform. Following this principle, we design a simple yet effective surrogate loss function, $h_{\btheta}$. We decompose the model $\w$ into two components: a body $\w_{bo}$ for representation extraction and a linear classifier head $\w_{he}$, with a similar decomposition for the anchors, $\ba_m = [\ba_{m}^{bo}; \ba_{m}^{he}]$. We utilize the representations generated by $\w_{bo}$ and replace the classifier head $\w_{he}$ with a new classifier head $\btheta$ to adapt to the labels of $\D_S’$. During training, we first optimize $\btheta$ to classify the labels of $\D_S’$ based on the representations extracted by $\w_{bo}$ (\cref{eq:popo}), and then search for the optimal coefficients of the anchors to update $\w_{bo}$ (\cref{eq:pqqq}):

\vspace{-1.em}
\begin{align}
    \label{eq:popo}
    \btheta^* &= \argmin_{\btheta}h(\D_S', [\w^k_{bo}+ \sum\nolimits_{m=1}^{|\A|}c_{m}\bar\ba_{m}^{bo}; \btheta]),\\
    \vspace{-.2em}
    \label{eq:pqqq}
    \hat{\boldsymbol c} &= \argmin_{\bc}h(\D_S', [\w^k_{bo}+ \sum\nolimits_{m=1}^{|\A|}c_m\bar\ba_{m}^{bo}; \btheta^*]),
\end{align}

\noindent{where} $\hat{\boldsymbol c}=[\hat c_1 , \ldots, \hat{c}_{|\A|}]$. Although $h_{\btheta^*}$ ignores $\w_{he}^k$ and $\{\ba_{m}^{he}\}_{m=1}^{|\A|}$ in \cref{eq:pqqq}, we find that the search results are a good indicator of the overall dimensions. Therefore, we update the full model via $\w^{k+1} = \w^k + \sum_{m=1}^{|\A|} \hat c_m\bar \ba_m $. Our experiments show that this approach works well even when $\D$ are dermoscopic images of skin lesions while $\D_S'$ is ImageNet with natural images. We note that the loss $h$ can potentially leverage unsupervised learning techniques \citep{caron2018deep,caron2020unsupervised}, allowing $\D_S’$ to consist of unlabeled data with richer resources. In this work, we focus on the supervised setting as a proof of concept.

\subsubsection{Fallback Mechanism}
\label{sec:method:fallback}
To enhance the robustness of \texttt{Feddle} and improve the chance of \cref{eq:adfd} being hold during the coefficient search, we introduce a fallback mechanism.
Specifically, we initialize the merging coefficients $\{c_m'\}_{m=1}^{|\A|}$ using an existing FL method and add a regularization term to the search objective. The resulting search objective becomes:
\vspace{-.5em}
\begin{equation}
    \label{eq:kiuy}
    \argmin_{\boldsymbol c}\ell (\D_S, \w^k + \sum_{m=1}^{|\A|}c_m\bar\ba_m) + \frac{\lambda}{2}
\sum_m^{|\A|} (c_m - c_m')^2,
\end{equation}

\noindent{where} $\lambda$ controls the regularization strength. In this work, we adopt \texttt{FedBuff} as the fallback method. However, \texttt{Feddle} can potentially leverage various FL methods as a fallback, which we leave for future exploration. As we will show in the experiment section, fallback mechanism is crucial for applying \texttt{Feddle} to the OOD setting.

\subsection{Theoretical Analysis}
\label{sec:method:analysis}
 Let ${F_j(\w) = \maE_{\D_j} [\ell(\D_j,\w)], \forall j=1,\ldots,J}$, we define the true loss on $\D$ as ${\maE_\D[\ell(\D,\w)]:=F(\w) = \frac{1}{J}\sum_{j=1}^J F_j(\w)}$. In the following, we show that due to the additional $\D_S$ and optimization of the merging coefficients, \texttt{Feddle} achieves a faster convergence rate in terms of communication rounds than existing methods (e.g. \texttt{FedAvg}, \texttt{FedBuff}) when ID data is available at the server. We further investigate the case when the server only has OOD data and show that \texttt{Feddle} remains convergent. Our theorems are based on the following assumptions which are widely used in the literature \citep{wang_obj_inconsistency,nguyen2022federated,wang2024fadas}. All the proofs are deferred to \cref{sec:app:proof}.
\begin{assumption}[Unbiased stochastic gradient]\label{ass:1}$\maE_{\xi_j}[g_j(\w;\xi_i)] = \nabla F_j(\w)$ for all $1\leq j\leq J$, where $\xi_i$ is the random variable for the noise and $g_j(\w;\xi_i)$ is stochastic gradient.
\end{assumption}
\begin{assumption}[Bounded local and global variance]\label{ass:2} For all $1\leq j\leq J$,
$$\maE_{\xi_j}[||g_j(\w;\xi_j)-\nabla F_j(\w)||^2] = \sigma^2_l(\w)\leq \sigma^2_l,$$ 
$$\frac{1}{J}\sum_{j=1}^J||\nabla F_j(\w) - \nabla F(\w)||^2 = \sigma^2_g(\w)\le\sigma^2_g,$$
{where $\sigma_l$ is the upper bound for the variance of the gradient due to the noise variable $\xi_i$ and $\sigma_g$ is the upper bound for the variance of the gradient due to heterogeneity.}
\end{assumption}
\begin{assumption}[Bounded gradient]\label{ass:3} $\exists {G \geq 0}, {||\nabla F_j||^2\leq G^2}$, for all $1\le j\le J$.
\end{assumption}
\begin{assumption}[L-smoothness]\label{ass:4} For all $1\le j\le J$,
    $$\exists L> 0,\; ||\nabla F_j(\w) - \nabla F_j(\w')||\le L||\w-\w'||.$$
    
\end{assumption}
\begin{theorem}[In-domain data]\label{thm:in-domain}
    Suppose the above assumptions hold, and $\D_S$ represents the in-domain data. In addition, suppose the client's delay is bounded by $\tau_{max}$, and Feddle's merging coefficients satisfies ${abs(\hat{c}_m)<\hat c_{max}}$. Then, Feddle at least has the same convergence rate as \texttt{FedBuff} and \texttt{FedAvg}, in $K$ global communication rounds, $Q$ local steps, and {$T$} server steps of training with the global step size (in \texttt{FedBuff} and \texttt{FedAvg}) ${\eta_g=\mathcal O(\sqrt{QM})}$, local step size $\eta_l = \mathcal O(1/\sqrt{K}Q)$, and server step size $\eta_c = \mathcal O(1/\sum_m^M||\Delta_m||^2)$, where $\Delta_m$ is model update of client $m$. Moreover, if the signal-to-noise ratio of the gradient is sufficiently large, i.e.\ ${C||\nabla F(\w)||^2\ge (\sigma^2_l(\w)+\sigma^2_g(\w))}$, for $C>0$, and for any delay $\tau\le\tau_{max}$,  there exists $C_{max}>0$ such that $C_{max}||\nabla F_j(\w^k)||^2\ge ||\nabla F_j(\w^{k-\tau})||^2$, then the convergence rate of Feddle $r_{Feddle}$ satisfies
    \begin{equation}\label{eq:feddle_rate}
        r_{Feddle}\le\frac{\sqrt{QMK}}{\sqrt{QMK} + C_T\left(K-\sqrt{\frac{M}{Q}}\right)}r_{FL},
    \end{equation}
    where $r_{FL}$ is the rate of \texttt{FedBuff} or \texttt{FedAvg}, ${C_T = A_0\left(1-\frac{1}{4^T}\right)}$, and $A_0$ is a constant decreasing in $C,C_{max},L$. Normally, $K\gg\sqrt{\frac{M}{Q}}$, thus \texttt{Feddle} has a faster convergence rate than the other methods.
\end{theorem}
\begin{remark}
    There are two assumptions in Theorem \ref{thm:in-domain} regarding the norm of the gradient. The first one asserts that the variance of the gradient can be bounded by some factor of the norm of the gradient. This is expected for reasonable training results and is also assumed in Assumption 4.3 of \cite{Bottou2016OptimizationMF}. The second assumption is technical where we would like to make the norm of the gradients comparable despite the delay. We note that the effect of $C_{max}$ is clarified in the constant $A_0$, where small $C_{max}$ provides smaller bounds, which is reasonable in practice.
\end{remark}

\begin{theorem}[Out-of-domain data]\label{thm:out-of-domain}
    Suppose the assumptions in Theorem \ref{thm:in-domain} hold except that ${\eta_{\boldsymbol{c}} = \min\left(\frac{1}{2LT||\boldsymbol{\Delta^k}||^2},\frac{1}{2LT||\boldsymbol{\Delta^k}||}\right)}$, where $\bm{\Delta^k}$ denotes all model updates at round $k$. In addition, the cosine similarity between $\partial h(\D'_S,\cdot)/\partial\boldsymbol{c}$ and $\partial \ell(\D,\cdot)/\partial\boldsymbol{c}$ is $s\approx1$, (c.f. Eq. (\ref{eq:adfd})), where $\D'_S$ is the OOD data. Then, if we choose the $\eta_{\boldsymbol{c}}'$ in the server training adaptively such that $\eta_{\boldsymbol{c}}'||\maE_{\D_S}[h(\D_S,\w^k)]|| = \eta_{\boldsymbol{c}}||\maE_{\D_S}[h(\D_S,\w^k)]||$ and let \texttt{FedBuff} initialize \texttt{Feddle} for the merging procedure, then it converges to a stable point of the true loss up to some error, 
    \begin{align}
        \nonumber
        r'_{Feddle}\le&\frac{\sqrt{QMK}}{\sqrt{QMK} + C_T'\left(K-\sqrt{\frac{M}{Q}}\right)}r_{FL} \\
        &+ \mathcal{O}\left((1-s)\frac{GK}{L}\right),
    \end{align}
    where $C_T'=\min\left\{\frac{A_0'\sqrt{K}}{T(\sigma_l + \sigma_g + G)},\frac{A_0'}{T}\right\}$ for some $A_0'>0$.
\end{theorem}
\begin{remark}
    Here, the rate is slightly different from that on ID data because we need to choose $\eta_{\boldsymbol{c}}$ adaptively to bound the error term due to the surrogate loss. In \cref{fig:ood_signal} of \cref{app:res:loss}, we studied the constant $s$ and showed empirically that it is indeed close to 1. Therefore, \cref{thm:out-of-domain} theoretically confirms the convergence result up to a small error.
\end{remark}

\begin{table*}[t!]
\centering
\resizebox{\textwidth}{!}{
\begin{tabular}{lcccccccccc}
\toprule[1.5pt]
\multirow{2}{*}{Method} & \multirow{2}{*}{ID} & \multicolumn{4}{c}{ResNet18} & \multicolumn{4}{c}{ViT} \\
\cmidrule(lr){3-6} \cmidrule(lr){7-10}
& & \footnotesize Dir(0.1), $\N(20)$ & \footnotesize Dir(0.1), $\N(5)$ & \footnotesize Dir(0.3), $\N(20)$ & \footnotesize Dir(0.3), $\N(5)$ & \footnotesize Dir(0.1), $\N(20)$ & \footnotesize Dir(0.1), $\N(5)$ & \footnotesize Dir(0.3), $\N(20)$ & \footnotesize Dir(0.3), $\N(5)$ \\
\specialrule{0em}{1pt}{1pt}
\hline
\specialrule{0em}{1pt}{1pt}
 Center & \tikzcmark
  & \multicolumn{4}{c}{$53.2\pm1.2$} 
  & \multicolumn{4}{c}{$70.1\pm0.8$} \\
 Fed$+$FT & \tikzcmark
  & \underline{$58.7\pm0.4$} 
  & \underline{$64.8\pm0.4$} 
  & \underline{$59.6\pm0.3$} 
  & \underline{$63.3\pm0.6$} 
  & \underline{$85.9\pm0.3$} 
  & \underline{$88.8\pm0.2$} 
  & \underline{$86.1\pm0.1$} 
  & \underline{$89.3\pm0.1$}\\
 HFCL & \tikzcmark
  & \underline{$62.2\pm0.9$} 
  & \underline{$68.3\pm0.3$} 
  & \underline{$63.8\pm0.6$} 
  & \underline{$68.6\pm0.2$} 
  & \underline{$86.4\pm0.5$} 
  & \underline{$88.5\pm0.1$} 
  & \underline{$86.3\pm0.0$} 
  & \underline{$88.7\pm0.1$}\\
 FedDF-ID & \tikzcmark
  & \underline{$55.0\pm0.5$} 
  & \underline{$65.4\pm0.2$} 
  & \underline{$59.8\pm0.8$} 
  & \underline{$68.0\pm0.2$} 
  & $63.7\pm0.5$ 
  & \underline{$81.9\pm0.2$} 
  & $66.6\pm0.9$ 
  & \underline{$84.9\pm0.2$}\\
 Feddle-ID (ours) & \tikzcmark
  & \underline{\bftab 72.4 $\pm$ 2.1} 
  & \underline{\bftab 74.3 $\pm$ 0.5} 
  & \underline{\bftab 76.5 $\pm$ 0.3} 
  & \underline{\bftab 77.2 $\pm$ 0.0} 
  & \underline{\bftab 90.6 $\pm$ 0.0} 
  & \underline{\bftab 90.0 $\pm$ 0.2} 
  & \underline{\bftab 92.5 $\pm$ 0.4} 
  & \underline{\bftab 92.5 $\pm$ 0.0}\\
\arrayrulecolor{lightgray}
\midrule
\arrayrulecolor{black}
 FedAvg &  \tikzxmark
  & $52.6\pm0.9$ 
  & \underline{$65.0\pm0.6$} 
  & \underline{$57.4\pm1.4$} 
  & \underline{$68.5\pm0.2$} 
  & $49.9\pm1.1$ 
  & \underline{$78.2\pm0.5$} 
  & $48.4\pm0.9$ 
  & \underline{$80.8\pm0.1$}\\
 FedAsync & \tikzxmark
  & \underline{$58.0\pm1.3$} 
  & \underline{$66.9\pm0.1$} 
  & \underline{$62.4\pm0.8$} 
  & \underline{$71.0\pm0.3$} 
  & $66.7\pm1.8$ 
  & \underline{$83.8\pm0.3$} 
  & $69.8\pm0.8$ 
  & \underline{$86.7\pm0.2$}\\
 FedBuff &  \tikzxmark
  & \underline{$64.6\pm0.4$} 
  & \underline{$66.4\pm0.3$} 
  & \underline{$68.8\pm0.4$} 
  & \underline{$69.8\pm0.2$} 
  & \underline{$86.7\pm0.6$} 
  & \underline{$86.8\pm0.5$} 
  & \underline{$88.6\pm0.6$} 
  & \underline{$89.7\pm0.3$}\\
 CA2FL &  \tikzxmark
  & \underline{$66.0\pm1.2$} 
  & \underline{$67.3\pm0.2$} 
  & \underline{$69.5\pm0.4$} 
  & \underline{$70.1\pm0.1$} 
  & \underline{$87.3\pm0.7$} 
  & \underline{$87.8\pm0.3$} 
  & \underline{$89.2\pm0.9$} 
  & \underline{$89.5\pm0.1$}\\
 FedDF-OOD &  \tikzxmark
  & $24.0\pm0.4$ 
  & $30.4\pm1.3$ 
  & $28.5\pm0.9$ 
  & $35.6\pm1.7$ 
  & $50.1\pm1.1$ 
  & \underline{$78.8\pm0.4$} 
  & $48.9\pm0.9$ 
  & \underline{$81.3\pm2.9$}\\
 Feddle-OOD (ours) &   \tikzxmark
  & \underline{\bftab 70.5 $\pm$ 1.8} 
  & \underline{\bftab 72.8 $\pm$ 0.6} 
  & \underline{\bftab 74.9 $\pm$ 0.8} 
  & \underline{\bftab 75.9 $\pm$ 0.0} 
  & \underline{\bftab 87.8 $\pm$ 0.5} 
  & \underline{\bftab 88.1 $\pm$ 0.6} 
  & \underline{\bftab 92.1 $\pm$ 0.1} 
  & \underline{\bftab 92.7 $\pm$ 0.3}\\
\bottomrule[1.5pt]
\end{tabular}
}
\caption{\textbf{Comparisons under various data heterogeneity and communication delay}. Two data heterogeneity levels (Dir(0.1), Dir(0.3)) and two delay levels ($\N(5), \N(20)$) are tested. Dataset is CIFAR100. ``ID" indicates whether the approach uses in-domain data. If so, 1000 samples are provided. Performance higher than \texttt{Center} is underlined. The best performance is highlighted by bold.}
\label{tab:results-resnet-vit}
\end{table*}
\vspace{-.5em}
\section{Experiments}
\label{sec:exp}
\vspace{-.5em}
We first introduce the models, baselines and server-side data in the OOD setting. Client datasets and other settings are detailed in the respective sections. Additional information, such as hyperparameters, is given in~\cref{app:setting}.

\noindent\textbf{Models.}  We employ ResNet-18~\citep{he2016deep} pretrained on ImageNet \citep{imagenet}, and perform full fine-tuning. Additionally, we apply LoRA~\citep{hu2021lora} fine-tuning to a ViT16-Base \cite{dosovitskiy2020vit}, which is also pretrained on ImageNet. Furthermore, we train a convolutional neural network (CNN) from scratch, with corresponding results presented in \cref{app:res:cnn}. 

\noindent\textbf{Baseline Methods.} Our primary focus is to address data heterogeneity and asynchronous communication in a hybrid data regime. While many prior studies tackle data heterogeneity as an isolated challenge, we compare \texttt{Feddle} against baselines that \textit{also account for asynchronous communication or hybrid data availability}, thereby aligning with our experimental setup. Specifically: \textbf{a)} For hybrid data regime, we compare with \textcolor{wacvblue}{(1)} \texttt{Center}, which trains the model exclusively on server-side data.  \textcolor{wacvblue}{(2)} \texttt{Fed+FL}, which fine-tunes aggregated model using server-side data at each communication round. \textcolor{wacvblue}{(3)} \texttt{HFCL}~\cite{elbir2021hybrid}, which trains a model on server-side data and aggregate it with client models. \textcolor{wacvblue}{(4)} \texttt{FedDF}~\cite{lin2020ensemble}, which distills the knowledge from client models into the global model using server-side data. Notably, \texttt{Center}, \texttt{Fed+FT} and \texttt{HFCL} require ID data , while \texttt{FedDF} can also be applied in an OOD setting. \textbf{b)} For asynchronous communication, we compare our method with several competitive asynchronous methods including \textcolor{wacvblue}{(1)} \texttt{FedAsync} \cite{xie2019asynchronous}, \textcolor{wacvblue}{(2)} \texttt{FedBuff} \cite{nguyen2022federated}, and \textcolor{wacvblue}{(3)} \texttt{CA2FL} \cite{wang2023tackling}.
Additionally, we include the classical FL approach \texttt{FedAvg}~\cite{fedavg} as a reference. 

We categorize these methods into two groups based on whether ID data is used at the server: \textbf{a)} with ID data, including \texttt{Center}, \texttt{Fed+FL}, \texttt{HFCL}, \texttt{FedDF-ID}, \texttt{Feddle-ID} (ours), and \textbf{(b)} without ID data, including \texttt{FedAvg}, \texttt{FedAsync}, \texttt{FedBuff}, \texttt{Ca2FL}, \texttt{FedDF-OOD}, \texttt{Feddle-OOD} (ours), where the latter two methods are capable of leveraging OOD data.

\noindent\textbf{OOD Setting.} We offer a subset of ImageNet~\cite{imagenet} with 250K images for \texttt{FedDF-OOD} and \texttt{Feddle-OOD} in the OOD setting. \textit{Notably, as ResNet18 and ViT are pretrained on the ImageNet,\texttt{FedDF-OOD} and \texttt{Feddle-OOD} do \textbf{not} use additional data information, but has different effectiveness in leveraging the data.}

\vspace{-.2em}
\subsection{Results}
\vspace{-.5em}
\label{sec:res}

\vspace{+.5em}
\noindent\textbf{Various Data Heterogenity and Communication Delay.}
We first compare methods under different data heterogeneity and asynchronous communication scenarios. We simulate two levels of data heterogeneity by partitioning the data using the Dirichlet distribution following \cite{yurochkin2019bayesian}, with parameters Dir(0.1) and Dir(0.3). This results in heterogeneous client data distribution in terms of the class label distribution and dataset size. Additionally, we model the delay for each client using a half-normal distribution $\N$, based on practical observations from \citep{nguyen2022federated}, with standard deviation of 5 and 20. We define a scenario with 500 clients, and 200 communication rounds. At each round, 10 clients are sampled. We perform image classification tasks using CIFAR-100~\cite{krizhevsky2009learning}, Results on additional datasets are given in \cref{additional-results}. For methods using ID data at the server, we provide only 1K samples, acknowledging the higher cost of collecting ID data. We repeat the experiments 3 times with different random seeds, and report the mean and standard deviation. 

\begin{figure}[t]
    \centering
    \begin{subfigure}[t]{.5\columnwidth}
        \centering
        \includegraphics[width=\linewidth]{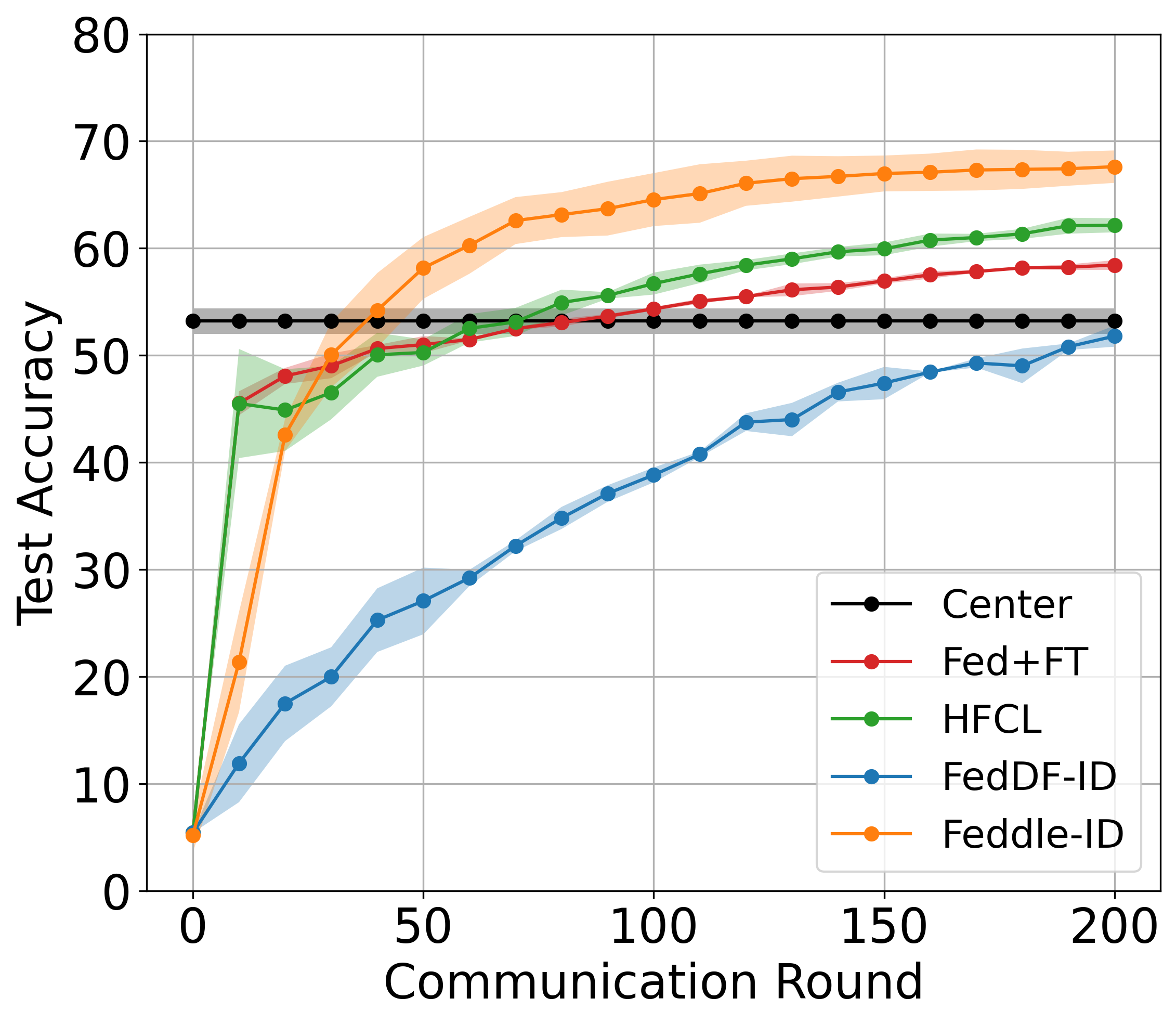}
        \caption{With In-Domain data.}
    \end{subfigure}%
    \begin{subfigure}[t]{.5\columnwidth}
        \centering
        \includegraphics[width=\linewidth]{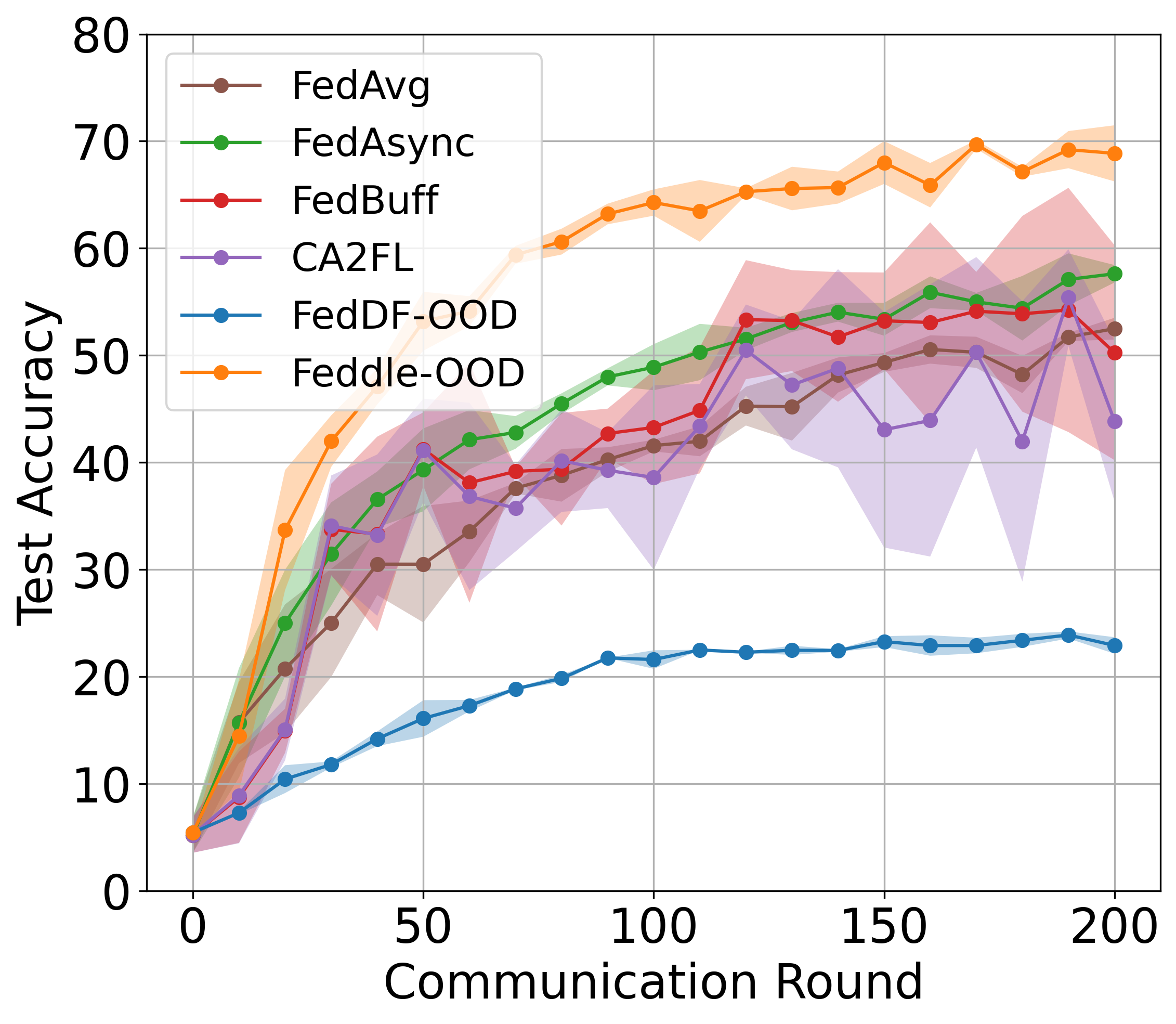}
        \caption{Without In-Domain data.}
    \end{subfigure}%
\caption{\textbf{Convergence plots} of ResNet18 on CIFAR100 with Dir(0.1), $\N(20)$. More plots are provided in \cref{app:res:convergence}.}
    \label{fig:converge-in-domain}
\end{figure}
In \cref{tab:results-resnet-vit}, we observe that \texttt{Feddle} consistently outperforms all the baseline methods by a clear margin regardless of whether ID data is available at the server. Overall, baseline methods exhibit poorer performance in scenarios characterized by strong heterogeneity (Dir(0.1)) and high delay ($\N(20)$) compared to simpler scenario (Dir(0.3) and $\N(5)$). In contrast, \texttt{Feddle} remains less impacted by heterogeneity or delay, achieving outstanding accuracy even in complex scenarios. 

\begin{table}[t!]
\centering
\begin{minipage}[t!]{0.7\linewidth}
\centering
\resizebox{\textwidth}{!}{
\begin{tabular}{lccc}
\toprule[1.5pt]
Method & ID & FEMNIST & CelebA \\
\specialrule{0em}{1pt}{1pt}
\hline
\specialrule{0em}{1pt}{1pt}
 Center &\tikzcmark  & $72.4\pm0.1$ & $75.8\pm0.3$  \\
Fed$+$FT &\tikzcmark & $82.4\pm0.9$& $84.3\pm0.1$ \\
 HFCL &\tikzcmark & $80.3\pm0.1$& $83.9\pm0.2$ \\
FedDF-ID &\tikzcmark & $81.1\pm0.6$& $83.8\pm0.3$\\
Feddle-ID (ours) &  \tikzcmark & \bftab 88.6 $\pm$ 0.1&  \bftab 90.2 $\pm$ 0.1\\
\arrayrulecolor{lightgray}
\cmidrule(lr){2-4}
\arrayrulecolor{black}
FedAvg & \tikzxmark & $74.2\pm0.1$& $74.3\pm1.0$ \\
FedAsync &  \tikzxmark& $85.2\pm0.2$& $85.9\pm0.7$\\
FedBuff &  \tikzxmark& $86.5\pm0.5$& $87.2\pm0.3$\\
CA2FL & \tikzxmark & $85.1\pm0.9$& $88.4\pm0.9$\\
FedDF-OOD &  \tikzxmark&$80.8\pm0.3$ & $81.9\pm1.2$\\
Feddle-OOD (ours) &  \tikzxmark & \bftab 88.5 $\pm$ 0.1& \bftab 90.2 $\pm$ 0.3\\
\bottomrule[1.5pt]
\end{tabular}
}
\end{minipage}
\hfill
\begin{minipage}[t!]{0.28\linewidth}
\captionof{table}{\textbf{Results of real-world data heterogeneity} on ResNet18. Delay level is set to $\N(20)$. ``ID" indicates whether the approach uses in-domain data. If so, 1000 samples are provided for FEMNIST and 200 for CelebA.}
\label{tab:results-resnet-realworld}
\end{minipage}
\end{table}

\begin{table}[t!]
\centering
\resizebox{0.7\linewidth}{!}{
\begin{tabular}{lcc}
\toprule[1.5pt]
Method & \footnotesize Dir(0.3), $\N(5)$ & \footnotesize Dir(0.3), $\N(20)$ \\
\specialrule{0em}{1pt}{1pt}
\hline
\specialrule{0em}{1pt}{1pt}
FedAvg & $52.4\pm1.2$& $59.0\pm1.5$ \\
FedAsync & $58.1\pm0.5$& $59.1\pm0.2$\\
FedBuff & $57.1\pm0.6$& $60.1\pm0.2$\\
CA2FL & $59.2\pm0.4$& $60.3\pm0.6$\\
FedDF-OOD & $53.2\pm0.2$ & $58.9\pm1.4$\\
Feddle-OOD (ours) & \bftab 60.6 $\pm$ 0.4& \bftab 62.8 $\pm$ 0.6\\
\bottomrule[1.5pt]
\end{tabular}
}
\caption{\textbf{Results of medical image dataset ISIC (2019)} on ResNet18 under two experiment settings. For FedDF-OOD and Feddle-OOD, ImageNET is provided to the server.}
\label{tab:results-medical}
\end{table}

\begin{table}[t!]
\centering
\resizebox{0.6\linewidth}{!}{
\begin{tabular}{lcc}
\toprule[1.5pt]
Method & $10\%$ & $20\%$ \\
\specialrule{0em}{1pt}{1pt}
\hline
\specialrule{0em}{1pt}{1pt}
Center & \multicolumn{2}{c}{$53.2\pm1.2$}\\
Fed+FT & $57.4 \pm 0.3$ & $57.4 \pm 0.7$ \\
HFCL & $57.6 \pm 0.2$ & $56.8 \pm 0.3$\\
FedDF-ID & $52.9 \pm 0.7$ & $51.2 \pm 0.5$\\
Feddle-ID (ours) & \bftab 68.6 $\pm$ \bftab $0.9$  & \bftab 62.3 $\pm$ \bftab $1.1$\\
\bottomrule[1.5pt]
\end{tabular}
}
\caption{\textbf{Results of noisy dataset} using CIFAR100 and ResNet18. $10\%$ and $30\%$ indicate the fraction of decentralized training data with random labels.}
\label{tab:results-noisy}
\end{table}
Moreover, for ID data, 1K samples constitute 1/50 of the total decentralized data. However, baseline FL methods may still struggle to surpass the performance of \texttt{Center}, which \textit{relies exclussively on server-side data}. Conversely, \texttt{Feddle} consistently achieve significantly higher performance than \texttt{Center}. Notably, \textit{this advantage persists even when only OOD data is available}.  These results underscore the effectiveness of our method in leveraging decentralized data with guidance from server-side data despite the challenges of data heterogeneity and communication delays. \Cref{fig:converge-in-domain} presents the convergence plots.

\noindent{\textbf{Real-World Data Heterogeneity.}}
To evaluate the performance of our method under real-world data heterogeneity. We incorporate two additional datasets: CelebA~\cite{liu2015faceattributes} and FEMNIST~\cite{cohen2017emnist}, partitioning the data according to the LEAF framework~\cite{caldas2018leaf}. This results in data distributions reflecting those of distinct real-world individuals. Additionally, we assess our method in a large-scale setting with 1000 clients, sampling 50 clients per round. The total number of communication rounds is set to 200. As shown in \cref{tab:results-resnet-realworld}, \texttt{Feddle} consistently outperforms the baseline methods, demonstrating its potential for real-world applications.

\noindent\textbf{Misaligned Domains Between Server and Client Data.}
To evaluate scenarios where the domain of server-side data differs from that of the clients, we conduct experiments on the ISIC 2019 with dermoscopic images of skin lesions. As shown in \cref{tab:results-medical}, \texttt{Feddle} outperforms all baselines even when using ImageNet as the server-side data, demonstrating promising generalization and broad applicability.

\noindent\textbf{Noisy Client Data.}
Since decentralized training data can be noisy in practice, we further evaluate the robustness of different methods under label noise. Specifically, we assess how well each method aggregates noisy model updates in the hybrid data setting. As shown in \cref{tab:results-noisy}, \texttt{Feddle} consistently achieves the best performance, highlighting its superior ability in mitigating the effects of noise through data-guided aggregation.

\noindent{\textbf{Analysis of Feddle Optimization.}}
In \cref{sec:bg}, we discuss that uniformly assigning positive aggregation coefficients may be suboptimal due to conflicting optimization directions across clients.  As \texttt{Feddle} searches for the optimal aggregation coefficients under data guidance and outperforms existing technologies, we show that the coefficient found by \texttt{Feddle} indeed contains negative values in 
\cref{fig:search_patter_dir01_n20}. Additionally, theoretical analysis in \cref{sec:method:analysis} shows that \texttt{Feddle} converges under OOD data availability if its gradient of the aggregation coefficients is aligned with the gradients when ID data is applied, namely satisfying \cref{eq:adfd}. As shown in \cref{fig:ood_signal_dir01_n20}, we find that without fallback initialization, the similarity between the optimization directions regarding ID and OOD data appears random. In contrast, with fallback initialization, the optimization directions regarding ID and OOD data become highly aligned, with the cosine similarity approaching 1, demonstrating the crucial role of the fallback mechanism in the success of \texttt{Feddle}. More plots of these analyses are provided in \cref{additional-results}.

\begin{figure}[t]
\centering
\begin{minipage}[t]{0.48\columnwidth}
    \centering
    \includegraphics[width=\linewidth]{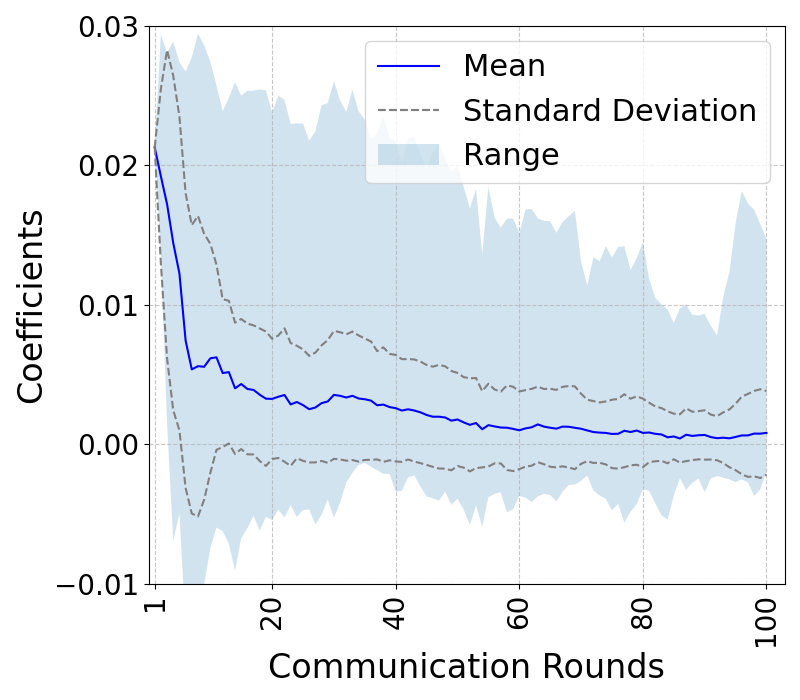}
    \caption{\textbf{Statistics of aggregation coefficients identified by \texttt{Feddle}} using ResNet and CIFAR10 with Dir(0.1), $\N(20)$.}
    \label{fig:search_patter_dir01_n20}
\end{minipage}
\hfill
\begin{minipage}[t]{0.48\columnwidth}
    \centering
    \includegraphics[width=\linewidth]{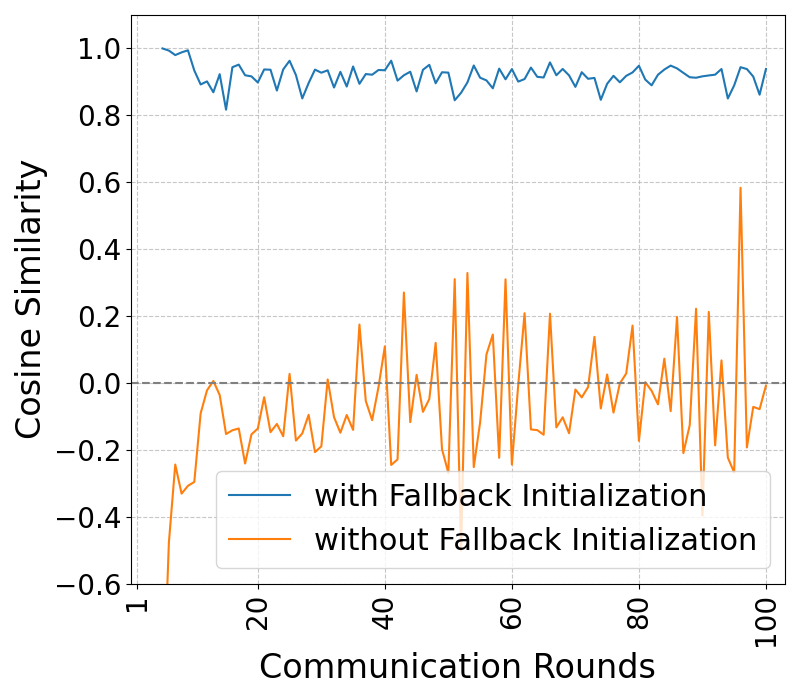}
    \caption{\textbf{Similarity of the coefficients' optimization direction between ID and OOD data} using ResNet18 and CIFAR10 with Dir(0.1), $\N(20)$.}
    \label{fig:ood_signal_dir01_n20}
\end{minipage}
\end{figure}
\begin{table}[t]
\centering
\resizebox{0.95\columnwidth}{!}{
\begin{tabular}{lcc}
\toprule[1.5pt]
& In-Domain & Out-of-Domain\\
\specialrule{0em}{1pt}{1pt}
\hline
\specialrule{0em}{1pt}{1pt}
Base & $81.9\pm0.7$ & $10.9\pm0.4$ \\
+Anchor Normalization & $83.1\pm0.4$ & 7.7 $\pm$ 0.4 \\
+Importance Score & $84.2\pm0.1$ & 9.3 $\pm$ 0.1 \\
+Fallback Initialization & \bftab 86.3 $\pm$ 0.4 & 81.0 $\pm$ 1.2 \\
+Fallback Regularization ($\lambda=0.01$) & \bftab 86.3 $\pm$ 0.2 & \bftab 82.2 $\pm$ 1.4 \\
\bottomrule[1.5pt]
\end{tabular}
}
\vspace{-.8em}
\caption{\textbf{Analysis of \texttt{Feddle} components for ID and OOD data availability} using ResNet18 and CIFAR-10 dataset.}
\label{tab:results-feddle-submodules}
\vspace{-1em}
\end{table}

\noindent\textbf{\texttt{Feddle} Components.} 
\Cref{tab:results-feddle-submodules} shows the effectiveness of \texttt{Feddle}'s components. The results under ID data availability highlight the benefit of each component in achieving superior performance, while fallback initialization is essential for handling OOD data. Further Ablation studies on the hyperparameters and comparison of computation complexities are given in \cref{sec:exp:abl}


\section{Conclusion}
In this work, we introduce the hybrid data regime, prevalent in real-world applications, and investigate methods to improve the utilization of decentralized data within this regime. Building on existing FL approaches that address the challenges of decentralized data, we propose a federated dual learning framework, \texttt{Feddle}. Our method provides a flexible solution for practical applications across various scenarios, as it can be applied whether the server data is ID or OOD relative to the clients' data. Theoretical analyses demonstrate that \texttt{Feddle} convergences faster than existing methodsm and experimental results confirm that it significantly outperforms current approaches, underscoring its effectiveness in training networks with decentralized data. Future directions are discussed in \cref{sec:fw}.
\label{sec:conclusion}

\section*{Acknowledgments}
MBB is funded by the Flemish Government under the Onderzoeksprogramma Artificiële Intelligentie (AI) Vlaanderen programme.

{
    \small
    \bibliographystyle{ieeenat_fullname}
    \bibliography{main}
}

\clearpage
\setcounter{page}{1}
\maketitlesupplementary
\appendix

\section{Proof}
\label{sec:app:proof}
In this section, we prove the convergence rate of \texttt{Feddle} when $\D_S$ is either the in-domain or out-of-domain data. For notational simplicity we let $\Delta^k_j = \Delta\w^k_{j}$ for the updates of the $j$-th client and $\boldsymbol{\Delta}^k = (\Delta^k_1,\dots,\Delta^k_M)$ for some $M>0$. We let $\nabla F$ denotes the usual gradient on $\w$ and ${\nabla_{\boldsymbol{c}} F(\w+\sum_{m=1}^M\hat c_m\Delta_m)}$ for the gradient with respect to $\boldsymbol{c}$ where $\Delta_m$ is the model update of client $m$.

We first prove a useful lemma for the smoothness of the gradient $\nabla F_{\boldsymbol{c}}$. 
\begin{lemma}\label{lem:eta_c}
    Suppose Assumptions \ref{ass:1}-\ref{ass:4} hold, then ${\tilde F(\bc):=F(\w+\sum_{m=1}^M\bc_m\Delta_m)}$ has $L(\sum_{m=1}^M||\Delta_m||^2)$-smoothness with respect to $\bc$. Therefore, for the convergence of the training on the server, $\eta_{\bc}\le\mathcal{O}\left(\frac{1}{L\sum_{m=1}^M||\Delta_m||^2}\right)$.
\end{lemma}
\begin{proof}
    Let $\w(\bc) = \w + \sum_{m=1}^M \bc_m\Delta_m$. Then, we have by definition that
\begin{align*}
    &||\nabla \tilde F_{\bc}(\bc) - \nabla \tilde F_{\bc}(\bc')||\\
    \le &||((\nabla F(\w(\bc))-\nabla F(\w(\bc'))\cdot\Delta_m)_m||\\
    \le&||\nabla F(\w(\bc))-\nabla F(\w(\bc'))||(\sum_m||\Delta_m||^2)^{1/2}\\
    \le &L||(c-c')\cdot(\Delta_1,\dots, \Delta_M)||(\sum_m||\Delta_m||^2)^{1/2}\\
    \le & L(\sum_m||\Delta_m|^2)||c-c'||.
\end{align*}
\end{proof}

\begin{proof}[Proof of Theorem \ref{thm:in-domain}]
    We first discuss the relation between our merging method $\sum_{m}\hat c_m\ba_m$ and that of FedBuff or \texttt{FedAvg}. For each $k$, let $\tilde\w^{k+1}$ be the weight update from FedBuff or \texttt{FedAvg} based on $\w^k$. In FedBuff, $\tilde\w^{k+1} = \frac{\eta_g}{M}\sum_{j\in \M}\Delta_j^{k-\tau^j_k}$ where $\M$ is the recent sampled set of clients and $\tau^j_k$ is the delay. Since \texttt{Feddle} has the same sampling procedure and these gradients from new clients are added to the atlas, there must exist certain $\hat c_m$ such that $\sum_{m}\hat c_m\ba_m = \tilde\w^{k+1}$. In other words, \texttt{Feddle}'s output in this step $\w^{k+1}$ must have a lower error than $\tilde\w^{k+1}$. In \texttt{FedAvg}, $\tilde\w^{k+1} = \frac{\eta_g|\D_j|}{|\D|}\sum_{j=1}^J\Delta_j^{k}$, where each client's gradient is involved. For this, we define $\hat c_m \propto \eta_g|\D_j|/|\D|$ (by $\propto$ we resume the unnormalized gradient) if $\ba_m$ is the most recent gradient of each client and otherwise $0$. By the $L$-smoothness assumption, the difference in the error at $\tilde\w^{k+1}$ and the defined $\w^k+\sum_{m}\hat c_m\ba_m$ can be bounded by $\sum_jL||\w^k-\w^{k-\tau_k^j}||$ (their gradients are calculated at $\w^k$,$\w^{k-\tau_k^j}$). According to the proof of Theorem 3.4 of \cite{wang2024fadas}, such difference can be bounded by Eq.(C.8), and their overall order is less than $\eta_g^2\eta^2_l$, which is therefore absorbed in  -$\eta_g\eta_l||\nabla F(\w^k)||^2$ for sufficiently small parameters. Thus, we can safely assume that the same conclusion holds when comparing \texttt{Feddle} and \texttt{FedAvg} up to a negligible error.  
    
    Now, we consider the expected loss on $\D_S$ (and thus $\D$), and $\D_j$, $1\le j\le J$, we have
    \begin{align*}
        F(\w^{k+1}) - F(\w^{k})&=  F(\w^{k+1}) - F(\tilde\w^{k+1}) \\
        &+ F(\tilde\w^{k+1}) - F(\w^{k}).
    \end{align*}
    Based on the proceeding argument, $F(\w^{k+1}) \le F(\tilde\w^{k+1})$ holds, thus we can assume without loss of generality that the optimization on $\boldsymbol{c}$ is initialized at the $\boldsymbol{c}$ such that $\tilde w^{k+1} = \w^k + \sum_m\hat c_m\ba_m$. By the existing convergence results of FL methods \cite{wang2024fadas}, we can upperbound $ F(\tilde\w^{k+1}) - F(\w^{k})$ by $-C_1\maE[||\nabla F(\w^k)||^2], C_1>0$ plus some constants regarding $T, K$. Thus, if we can upper bound $F(\w^{k+1}) - F(\tilde\w^{k+1})$ by $-C_2\maE[||\nabla F(\w^k)||^2], C_2>0$, then we can indeed show that the convergence rate of \texttt{Feddle} is faster at least by some factor. Now, we investigate the optimization of $\boldsymbol{c}$ in Eq.(\ref{eq:kkkk}). Since it uses SGD and the $L$- smoothness holds, we can upperbound the loss reduction by the sum of the squared $L^2$ norm of the gradients at each $\boldsymbol c$. In particular, we have
    \begin{align*}
        F(\w^{k+1}) - F(\tilde\w^{k+1}) &\le  F(\w^{k+1}_{\boldsymbol{c}_1}) - F(\tilde\w^{k+1})\\
        &\le -\frac{\eta_{\boldsymbol{c}}}{2}||\nabla_{\boldsymbol{c}} F(\tilde\w^{k+1})||^2 +\frac{\eta_{\boldsymbol{c}}^2\sigma_g'^2}{2},
    \end{align*}
    where $\w^{k+1}_{\boldsymbol{c}_1}$ is the weight update after the first step of the optimization initialized by $\tilde \w^{k+1}$, and $\sigma_g'^2$ is the noise of the gradient, which we can also assume to be controlled by the squared norm of the gradient. Next, we would like to relate the gradient regarding $\boldsymbol{c}$ with the gradient at $\w^k$ to estimate $C_2$ as previously discussed. We assumed that $\tilde\w^{k+1} = \frac{\eta_g}{J}\sum_{j\in J}\Delta_j^k$ for some $J>0$, which is a unified expression for FL methods (if $M$ anchors are actually used, simply let some $\Delta^k_j=0$ and change $J$ to $M$ in the following). By the chain rule of differentiation, we have that
    \begin{align*}
    &||\nabla_{\boldsymbol{c}}F(\tilde\w^{k+1})||^2 \\= &||(\nabla F(\tilde\w^{k+1})^\top\Delta_1,\dots,F(\tilde\w^{k+1})^\top\Delta_{J})||^2\\
        =&\sum_{j = 1}^{J}(\langle\nabla F(\w^{k}),\Delta_j\rangle +   \langle\nabla F(\tilde\w^{k+1})-\nabla F(\w^{k}),\Delta_j\rangle)^2\\
        \ge &\sum_{j=1}^{J}(\langle\nabla F(\w^k),\Delta_j\rangle)^2\\
        &- {{2\eta_l} L}||\nabla F(\w^k)||||\tilde \w^{k+1} - \w^k||||\boldsymbol{\Delta^k}||^2\\
        \ge &\frac{J}{\eta_g^2}(\langle\nabla F(\w^k),\frac{\eta_g}{J}\sum_{j=1}^J\Delta_j)\rangle^2 - \\&{{2\eta_l} L}||\nabla F(\w^k)||||\tilde \w^{k+1} - \w^k||||\boldsymbol{\Delta^k}||^2,
    \end{align*}
    where the last inequality used $J\sum_j a_j^2\ge(\sum_j a_j)^2$. Since $\frac{\eta_g}{J}\sum_{j=1}^J\Delta_j^k$ is the update in the FedBuff (or \texttt{FedAvg} depending on the mechanism), we can apply their theory to calculate the lower bound of the first term on the right-hand side. In particular, by Eq. (C.2)-(C.6) of \cite{wang2024fadas} and our condition on the gradient norm, we have that 
    \begin{align*}
        &|\langle\nabla F(\w^k),\frac{\eta_g}{J}\sum_{j=1}^J\Delta_j\rangle|
        \ge\frac{\eta_g\eta_lQ}{2}||\nabla F(\w^k)||^2 \\
        -& \frac{A_1\eta_g\eta_lQ^2\eta_l^2L^2(C+1)}{J}(\sum_{q,j}(||\nabla F(\w^{k-\tau_k^j})||^2\\
        +& 4M \hat c_{max}^2 \tau^2_{max}\max_{k-\tau^j_k\le s\le k}||\nabla F_j(\w^s)||^2))\\
        \ge&\frac{\eta_g\eta_lQ}{2}||\nabla F(\w^k)||^2\times\\
        &\left(1 - {A_1Q^2\eta_l^2L^2(C+1)C_{max}}(1+4\hat c_{max}^2 \tau^2_{max})\right)
    \end{align*}
    for some constants $A_1$. Here, we used $\frac{1}{J}\sum_j||\nabla F_j(\w)||^2\le\frac{2}{J}\sum_j||\nabla F_j(\w) - \nabla F(\w)||^2 + 2||\nabla F(\w)||^2$ in the last inequality. Note that $\eta_lQ$ is $\mathcal{O}\left(\frac{1}{\sqrt{T}}\right)$, and $\hat c_{max}\tau_{max}=\mathcal O(1)$, so the coefficients must be positive. Now, for the second term, we have by Eq.(C.6) of \cite{wang2024fadas} and our condition on the signal-to-noise ratio that 
    \begin{align*}
        &||\tilde \w^{k+1} - \w^k||^2\\
        \le &\frac{2(C+1)Q\eta^2_g\eta^2_l}{M}\sum_{q=0}^{Q-1}||\nabla F_j(\w^{k-\tau_k^j,q})||^2\\
        \le&\frac{2(C+1)Q\eta^2_g\eta^2_l}{M}\sum_{q=0}^{Q-1}(1+C\eta_lLG)^{2q}\sum_{j}||\nabla F_j(\w^{k-\tau_k})||^2\\
         \le&\frac{2(C+1)C_{max}Q\eta^2_g\eta^2_l}{M}\sum_{q=0}^{Q-1}(1+C\eta_lLG)^{2q}\sum_{j}||\nabla F_j(\w^k)||^2\\
        \le&\frac{2(C+1)C_{max}Q\eta^2_g\eta^2_l}{M}\sum_{q=0}^{Q-1}(1+C\eta_lLG)^{2q}\sum_{j}||\nabla F_j(\w^k)\\
        &-\nabla F(\w^k)+\nabla F(\w^k)||^2\\
        \le&{A_2(C+1)^2C_{max}Q\eta^2_g\eta^2_l}||\nabla F(\w^k)||^2.
    \end{align*}
    Here, $A_2$ is a constant. The second inequality is due to the fact that each $\Delta_{j,q}$ is obtained by local SGD, and therefore, the changes in the gradient can be bounded by the initial value and some constants. The last inequality is due to Assumption \ref{ass:2}, the low signal-to-noise ratio of the gradient and the order of $\eta_l$. In total, we have by choosing suitable coefficients for $\eta_l$ that
    \begin{align}\label{eq:app:gradient_bound}
        \nonumber&||\nabla_{\boldsymbol{c}}F(\tilde\w^{k+1})||^2\ge\frac{J\eta_l^2Q^2}{4}||\nabla F(\w^k)||^4 \\ 
        \nonumber -&2\eta_l^2\eta_gQL(C+1)\sqrt{A_2C_{max}
        }||\nabla F(\w^k)||^2||\boldsymbol{\Delta^k}||^2 \\
        \nonumber \ge& \left(\frac{A_3}{C_{max}}  - 2\eta_l^2\eta_gQL(C+1)\sqrt{A_2C_{max}
        }\right)\\
        &\cdot ||\nabla F(\w^k)||^2||\boldsymbol{\Delta^k}||^2.
    \end{align}
    Here, $A_3$ is a constant. We used Eq. (C.5) of \cite{wang2024fadas} and our conditions on the gradient norm for the last inequality, which yield that $C_{max}\eta_l^2Q^2J||\nabla F(\w^k)||^2\ge A_3||\boldsymbol{\Delta^k}||^2$. $\eta^2_l\eta_gQ=\mathcal{O}\left(\frac{M}{K}\right)$ which is assured to vanish given the convergence rate of federated learning algorithms, so the coefficients would be positive.
    
    Since $||\boldsymbol{\Delta^k}||^2 = \sum_j||\Delta_j||^2$, the decrease in the true loss is expected to be a constant factor of a function of $||\nabla F(\w^k)||^2$ and server steps $T$. Note that by using SGD under $L'$-smoothness, the norm of the gradient is at least $(1-L'(1+C)\eta_{\boldsymbol{c}})^T||\nabla F_{\boldsymbol{c}}(\tilde\w^{k+1})||$, therefore, $F(\w^{k+1}) - F(\w^k)$ is at least be upper bounded by
    \begin{align}\label{eq:app:A_0}
        -\eta_{\boldsymbol{c}}\frac{1-(1-L'(1+C)\eta_{\boldsymbol{c}})^{2T}}{1-(1-L'(1+C)\eta_{\boldsymbol{c}})^2}||\nabla F_{\boldsymbol{c}}(\tilde\w^{k+1})||^2.
    \end{align}
    By the choice of $\eta_{\boldsymbol{c}}$ in \Cref{lem:eta_c} and our lower bound on $||\nabla F_{\boldsymbol c}(\tilde \w^{k+1})||^2$ in \Cref{eq:app:gradient_bound}, we can upper bound the above by $-A_0(1-\eta_l^2\eta_gQ)\left(1-\frac{1}{4^{{T}}}\right)||\nabla F_(\w^{k})||^2$, where $A_0$ is a constant decreasing in $C,C_{max},L$. Therefore, the effective constants in front of $||\nabla F(\w^k)||^2$ is at least of order $-A_0(1-\eta_l^2\eta_gQ)\left(1-\frac{1}{4^{{T}}}\right)/2 - \eta_g\eta_lQ/2$ where the second constant comes from FedBuff \cite{wang2024fadas}. We obtain the results by plugging in the parameters.


\end{proof}

\begin{proof}[Proof of Theorem \ref{thm:out-of-domain}]
    Denote the expected surrogate loss by $F^h$, and the weight update after the $t$-th server step on $F^h$ and $F$ (i.e. out-of-domain and in-domain data) by $\w^k_{h^t},\w^k_{f^t}$. We first have 
    \begin{equation*}
        ||F(\w^k_{h^t}) - F(\w^k_{f^t})||\le G||\w^k_{h^t}-\w^k_{f^t}||.
    \end{equation*}
    Since we assume that \texttt{Feddle} is initialized by FedBuff when the server data is out-of-domain, we can let $\w^k_{h^0}=\w^k_{f^0}$ at the start point. Now, for $t=1$, by the definition of $\eta'_{\boldsymbol{c}}$ and the cosine similarity between the gradients, it holds that
    \begin{align*}
        &||\w^k_{h^1}-\w^k_{f^1}|| = ||\w^k_{h^1}-\w^k_{h^0}+\w^k_{f^0}-\w^k_{f^1}||\\
        =&\eta_{\boldsymbol{c}}||\eta_{\boldsymbol{c}}'\nabla_{\boldsymbol{c}} F^h(\w^k_{h^0})/\eta_{\boldsymbol{c}} - \nabla_{\boldsymbol{c}} F(\w^k_{f^0})||\\
        \le&\eta_{\boldsymbol{c}}(1-s)||\nabla_{\boldsymbol{c}} F(\w^k_{f^0})||\\
        \le&\eta_{\boldsymbol{c}}(1-s)||\nabla F(\w^k_{f^0})||||\boldsymbol{\Delta^k}||\ll1.
    \end{align*}
    Here, the last equality follows from the definition of the gradient with respect to the merging coefficients and Cauchy-Schwartz inequality.  In general, we have 
    \begin{align*}
        &||\w^k_{h^t}-\w^k_{f^t}|| \\
        =& ||\eta_{\boldsymbol{c}}'\nabla_{\boldsymbol{c}} F^h(\w^k_{h^{t-1}})-\eta_{\boldsymbol{c}}\nabla_{\boldsymbol{c}} F(\w^k_{f^{t-1}}) - \w^{k}_{f^{t-1}}+\w^k_{h^{t-1}}||\\
        \le&\eta_{\boldsymbol{c}}||\eta_{\boldsymbol{c}}'\nabla_{\boldsymbol{c}} F^h(\w^k_{h^{t-1}})/\eta_{\boldsymbol{c}} - \nabla_{\boldsymbol{c}} F(\w^k_{h^{t-1}})||\\
        &+||\w^k_{f^{t-1}}-\w^k_{h^{t-1}}||+\eta_{\boldsymbol{c}}||\nabla F(\w^k_{h^{t-1}}) - \nabla F(\w^k_{f^{t-1}})||\\
        \le&\eta_{\boldsymbol{c}}(1-s)||\nabla_{\boldsymbol{c}} F(\w^k_{h^{t-1}})||\\
        &+(1+L||\boldsymbol{\Delta^k}||^2\eta_{\boldsymbol{c}})||\w^k_{f^{t-1}}-\w^k_{h^{t-1}}||\\
        \le&\eta_{\boldsymbol{c}}(1-s)G||\boldsymbol{\Delta^k}||+(1+L||\boldsymbol{\Delta^k}||^2\eta_{\boldsymbol{c}})||\w^k_{f^{t-1}}-\w^k_{h^{t-1}}||
    \end{align*}
    And therefore by iteratively applying the formula, we have,
    \begin{align*}
        &||\w^k_{h^t}-\w^k_{f^t}||\\
        \le&\eta_{\boldsymbol{c}}(1-s)[G||\boldsymbol{\Delta^k}||\sum_{t=0}^{T-1}(1+L||\boldsymbol{\Delta^k}||^2\eta_{\boldsymbol{c}})^{t}\\
        &+(1+L||\boldsymbol{\Delta^k}||^2\eta_{\boldsymbol{c}})^{t}||\nabla_{\boldsymbol{c}} F(\w^k_{f^0})||]\\
        \le &\eta_{\boldsymbol{c}}(1-s)G||\boldsymbol{\Delta^k}||\sum_{t=0}^{T}(1+L||\boldsymbol{\Delta^k}||^2\eta_{\boldsymbol{c}})^{t}
    \end{align*}
    Therefore, we have an upper bound for $F(\w^k_{h^t})$ by
    \begin{equation*}
        F(\w^k_{h^t})\le F(\w^k_{f^t}) +\eta_{\boldsymbol{c}}(1-s)G||\boldsymbol{\Delta^k}||\sum_{t=0}^{T}(1+L||\boldsymbol{\Delta^k}||^2\eta_{\boldsymbol{c}})^{t} 
    \end{equation*}
    which means that
    \begin{align*}
        &F(\w^{k+1}_h) - F(\w^k)\le \\&F(\w^{k+1}) - F(\w^k) + \eta_{\boldsymbol{c}}(1-s)G||\boldsymbol{\Delta^k}||\sum_{t=0}^{T}(1+L||\boldsymbol{\Delta^k}||^2\eta_{\boldsymbol{c}})^{t} 
    \end{align*}
    Now, by choosing $\eta_{\boldsymbol{c}}$ adaptively such that $\eta_{\boldsymbol{c}} = \min\left(\frac{1}{LT||\boldsymbol{\Delta^k}||^2},\frac{1}{LT||\boldsymbol{\Delta^k}||}\right)$, we have, no matter the magnitude of $||\boldsymbol{\Delta^k}||$, that
    \begin{align*}
        &F(\w^{k+1}_h) - F(\w^k) \\
        \le&F(\w^{k+1}) - F(\w^k) + (1-s)\frac{A_4G}{TL}T\\
        =&F(\w^{k+1}) - F(\w^k) + A_4(1-s)\frac{G}{L}.
    \end{align*}
    with a potential price that the $C_T$ in the rate of \texttt{Feddle} changes to 
    \begin{equation*}
            C'_T = \frac{A_0}{T||\boldsymbol{\Delta^k}||}\ge\frac{A_5}{TQ\eta_l(\sigma_l + \sigma_g + G)}\ge\frac{A_0'\sqrt{K}}{T(\sigma_l + \sigma_g + G)}, 
    \end{equation*}
due to the change of $\eta_{\boldsymbol{c}}$ in \Cref{eq:app:A_0}, which still guarantees faster convergence. Therefore, we obtain the result by aggregating the loss reduction.
\end{proof}
\section{Algorithm}
\label{app:alg}
The algorithm of \texttt{Feddle} is presented in \cref{alg:server}.
\begin{algorithm}[t]
  \caption{\texttt{Feddle}}\label{alg:server}
    \begin{algorithmic}[1]
    \Require Server learning rate $\eta_g$, client learning rate $\eta_{\ell}$, number of server epochs $E_g$, number of client epochs $E_{\ell}$, atlas size $M$, number of FL rounds $K$. 
    \Ensure model $\w$
    \State{Initialize $\A=\{\}$}
    
    \Repeat{}
    \State{$j \leftarrow$ Sample available client}
        \State{Start the $j$th client training with $\{w^k, \eta_{\ell}, E_{\ell}\}$} 

        \If{receive client update}
            \State{$\Delta_j \leftarrow$ Model update from the $j$th client}
            \State{$\{\ba_m\} \leftarrow$ UpdateAtlas($\{\ba_m\}$, $\{s_m\}$, $\Delta_j$)}

        \EndIf
    
        \If{time to start coefficient search at round $k$}
            \State $\{\hat c_m\}$, $\w^{k+1}\leftarrow \text{\small GlobalModelUpdate}(\{\ba_m\}, \w^k)$
            \State {$s_m \leftarrow |\hat c_m|, \quad \forall m=1,\ldots, |\A|$.}
        \EndIf
        
    \Until{Stopped}
    
    \Function{UpdateAtlas}{$\{\ba_m\}$, $\{s_m\}$, $\Delta_j$}
   
        \If{$|\A| < M$} 
            \State $m' \leftarrow |\A|+1$
        \ElsIf{$|\A| == M$} 
            \State {$m'=\argmin_m{\{s_m\}}$}
        \EndIf
        
        \State $\ba_{m'} \leftarrow \Delta_j$

        \State\Return $\{\ba_m\}$
    
  \EndFunction

   \Function{GlobalModelUpdate}{$\{\ba_m\}$, $\w^k$}
   
        \State {$\{\bar\ba_m\} \leftarrow \text{Normalize}(\{\ba_m\})$ using \cref{eq:anc_normalization}}
        
        \State {Initialize $\{c_m\}$ as $\boldsymbol{0}$ or with respect to Fallback}
        
        \State {$\{\hat c_m\} \leftarrow$ Coefficient search using \cref{eq:kkkk} or \cref{eq:pqqq}}
        
        \State {$\w^{k+1} = \w^k + \sum_{m=1}^{|A|} \hat c_m\bar \ba_m$}

        \State\Return $\{\hat c_m\}$, $\w^{k+1}$
    
  \EndFunction
    \end{algorithmic}
\end{algorithm}
\section{Related Work}
\label{sec:rw}
We begin by discussing related work on learning from decentralized data (\cref{sec:rw:decent}), with a particular emphasis on Federated Learning (FL). Next, we examine related work on learning from hybrid data in \cref{sec:rw:hybrid}. Finally, we discuss model averaging methods from a broad scope in \cref{sec:rw:merg}.
\subsection{Learning from Decentralized Data}
\label{sec:rw:decent}
Next, we proceed to discuss learning from decentralized data, starting with related work on FL. We then briefly introduce other decentralized learning schemes.

\paragraph{Federated Learning}
The pioneering FL framework was proposed by \citet{fedavg}, which features a client-server architecture. In this setup, the server orchestrates the
training process by sending the latest global model to connected clients for local training and aggregating their model updates once local training is completed at each client. The clients are responsible for managing local training using their own data and computational resources.

\begin{itemize}[leftmargin=8pt,itemsep=0pt,topsep=0pt]
\item \textit{Data Heterogeneity}. One major challenge in FL is data heterogeneity, where clients are often geographically distant or environmentally distinct from each other~\citep{kairouz2021advances,fedprox}. This heterogeneity can lead to multiple optimization iterations during local training, causing convergence challenges for FL methods~\citep{wang_obj_inconsistency,Li2020On}. To address this issue, various enhancement strategies have been proposed. One approach involves using Bayesian modeling to capture and regularize the correlation and variation between clients~\citep{chen2021fedbe,yurochkin2019bayesian,zhang_pfedbayes,Zhu_2023_CVPR}. Another method uses a variance reduction model to correct for client-drift in local updates~\citep{scaffold}. Additionally, researchers have explored contrastive learning on shared representations~\citep{li2021model} and guided local training using learned local drift~\citep{gao2022feddc}. A different stream of work adopts a divide-and-conquer strategy, identifying clusters of data distribution among clients to make the federated learning task more focused on the data within each cluster~\citep{ifca}. In some cases, the training target is even tailored to individual client's data distribution, with information shared
within the FL framework used to obtain a better initialization for local adaptation~\citep{zhang2021parameterized,song2022personalized,nikoloutsopoulos2022,dinh_pfedme}. Notably, these methods do not focus on learning a global model that processes the entire data distribution.

\item \textit{Asynchronous Communication.}
The previous works discussed above consider synchronous communication, where the server waits for all selected clients to complete their local training and report model updates before aggregation. However, this setup can lead to long wall-clock times due to the slowest client's downlink, local training, and uplink delays~\citep{bonawitz2019}. To address this issue, researchers have proposed various strategies for handling asynchronous communication. One approach involves sampling a large number of clients and aggregating model updates as soon as a minimum threshold is reached, while delayed clients are discarded~\citep{bonawitz2019}. However, this method can lead to skewed population data distributions observed by the server, especially when clients have inherent and consistent delay. Moreover, abandoned clients may still complete their local
training, resulting in excessive energy consumption. This approach is also not applicable if client groups are not large-scale (e.g., hundreds or thousands). To mitigate these issues, some researchers have proposed reducing the local training workload of slow clients~\citep{zhang2023timelyfl} or discarding model updates that exceed a predefined delay threshold (could be multiple communication rounds)~\citep{liu2024fedasmu}. However, they are not able to comprehensively address the challenges of asynchronous communication. Another line of work focuses on developing algorithms that can accept and utilize model updates regardless of their delay, allowing for more efficient asynchronous communication~\citep{xie2019asynchronous,nguyen2022federated,wang2023tackling,wang2024fadas,leconte2024queuing}, which are the main baselines we compare with in this paper.
\end{itemize}
\vspace{-1em}
\paragraph{Alternative Decentralized Learning Schemes.}
Beyond FL frameworks, there are other paradigms that support decentralized training. For instance, Gossip Learning \citep{ormandi2013gossip, robustdecentralized} facilitates peer-to-peer communication among clients, eliminating the need for a central server. In this setup, distributed clients exchange and aggregate model updates directly at their local computation nodes. However, such methods typically exhibit lower efficiency compared to FL due to the lack of a centralized coordination. Given that our work focuses on hybrid data regimes where a natural center node exists, we concentrate on leveraging the FL mechanism to learn from decentralized data.

\subsection{Learning from Hybrid Data}
\label{sec:rw:hybrid}
Prior studies have explored leveraging server-side data to enhance decentralized learning. These methods consider scenarios where clients lack  sufficient computational resources for local training and instead upload their data to the server \citep{elbir2021hybrid,ni2023semi,feng2023hybrid}. The server computes model updates on the behalf of these clients while coordinating federated learning among other clients. Additionally,  work incorporating knowledge distillation into FL frameworks \citep{li2019fedmd,lin2020ensemble,yang2023fedfed} leverages (collected or synthesized) server-side data to integrate clients' knowledge into the global model. However, these approaches do not address the challenge of asynchronous communication, as knowledge is equally extracted from every client regardless of potential communication delays. In contrast, our method tackles more practical challenges and can be applied to a broader range of scenarios thanks to the accommodation of out-of-domain data availability. 

Notably, \citet{yueqioptimizing} propose a method that searches for optimal merging coefficients, similar to our approach. However, their technique is limited to searching within the convex hull of reported models, whereas we demonstrate that optimal coefficients can be negative. Furthermore, our approach can accommodate out-of-domain data availability scenarios where server-side data is visually distinct from decentralized data, whereas these previous works are restricted to in-domain data availability.

\subsection{Model Averaging}
\label{sec:rw:merg}
Beyond model aggregation at the server in federated learning, model averaging has been explored in other areas as well. For instance, \citet{pmlr-v162-wortsman22a} demonstrate the potential of averaging models fine-tuned with diverse hyperparameter configurations. Several works
\citep{ilharco2023editing,dataless,yadav2023tiesmerging,mario} have proposed advanced averaging strategies to merge models fine-tuned for different downstream tasks in language modeling, aiming to create a new model with multiple capabilities. However, these approaches often overlook data heterogeneity, which can lead to conflicting information between different weights (or model updates), as discussed in \cref{sec:bg}. In image generation, significant oscillation has been observed during the training of diffusion models. To address this, \citet{liu2024linear} propose searching for optimal coefficients to merge all historical model weights. However, such averaging methods typically do not involve a cyclic optimization process between server and client, which means they also do not investigate the asynchronous communication challenge. This distinction sets our contribution apart from these related works.
\section{Experimental Setting}
\label{app:setting}

\paragraph{Models.}
In this paper, we conduct experiments on three network architectures: \textbf{1)} A small Convolutional Neural Network (CNN), consisting of three convolutional layers, one pooling layer, and two fully connected layers. \textbf{2)} A ResNet18~\cite{he2016deep} model pre-trained on ImageNet~\citep{imagenet}, using the checkpoint shared by PyTorch.\footnote{Source: \url{https://download.pytorch.org/models/resnet18-f37072fd.pth}} \textbf{3)} A Vision Transformer (ViT16-Base)~\citep{dosovitskiy2020vit} model pre-trained on ImageNet~\citep{imagenet}, using the checkpoint provided by \citet{jia2022vpt}.\footnote{Source: \url{https://github.com/KMnP/vpt}}
During local training, clients update all parameters of the CNN and ResNet18 models. For ViT16-Base, we apply Low-rank Adaptation (LoRA)~\cite{hu2021lora} with a rank of 4 and an adaptation scale of 8. Input images are resized to $32\times 32$ for the CNN, $224\times 224$ for ResNet18, and $224\times 224$ for ViT16-Base.

\paragraph{Datasets.}
We employ three datasets to simulated data heterogeneity: CIFAR10/100 \cite{krizhevsky2009learning}, and Fashion-MNIST \cite{xiao2017/online}. For CIFAR-100, the network is trained to classify 20 superclasses, while for the other datasets, it predicts their respective 10 classes.
Following previous work~\citep{yurochkin2019bayesian}, we partition the training data among clients using a Dirichlet distribution with two concentration parameters: Dir(0.1) and Dir(0.3). For CIFAR100, we perform partitioning with respect to the 100 fine classes, which can induce label concept drift \citep{Zhu_2023_CVPR} and increase data heterogeneity among clients. For the other datasets, partitioning is based on their respective 10 classes.
For the in-domain data availability, we reserve 1000 data points from each dataset's test set as server-side data. The remaining test data is used to evaluate the test accuracy. For the out-of-domain data availability, we use ImageNet~\citep{imagenet} as the server-side dataset. Given that out-of-domain data often has abundant resources, we create a subset of ImageNet with 250K data points.

Additionally, we utilize two dataset containing real-world data heterogeneity: FEMNIST~\cite{cohen2017emnist} and CelebA~\cite{liu2015faceattributes}. We partition the data by real-world individual following the LEAF framwork~\cite{caldas2018leaf}, and then randomly sample 1000 clients to form the federated learning group. For FEMNIST, we restrict the task to be predicting the first 40 classes (removing data corresponding to other classes). While for CelebA the task is to predict the "Smiling" attribute.

\paragraph{Methods.}
We compare our method \texttt{Feddle} with several competitive federated learning approaches: \textbf{1)} hybrid approaches \texttt{Fed+FT}, \texttt{HCFL}~\cite{elbir2021hybrid}, \texttt{FedDF}~\cite{lin2020ensemble}, and \textbf{2)} asynchronous methods \texttt{FedAsync} \cite{xie2019asynchronous}, \texttt{FedBuff} \cite{nguyen2022federated},
and \texttt{CA2FL} \cite{wang2023tackling}.
Additionally, we include \texttt{FedAvg} \cite{fedavg} and \texttt{Center}, which is trained exclusively on the server-side data. The learning rate and number of training epochs for the clients are tuned with Adam optimizer~\citep{kingma2014adam} using \texttt{FedAvg}, and these settings are subsequently applied to all other methods.
Hyperparameter tuning is performed for each method using CIFAR10 as a representative dataset. The optimal hyperparameters are then applied to the other datasets, as our experiments show that they remain largely consistent across different datasets with the same network architecture.

The hyperparameters searched for each method are as follows:
\begin{itemize}[leftmargin=8pt,itemsep=0pt,topsep=0pt]
\item \texttt{Fed+FT}: Server learning rate $\eta\in\{1e-5, 1e-4, 1e-3\}$, Server epochs $M\in\{1, 10, 20\}$.
\item \texttt{FedDF}: Server learning rate $\eta\in\{1e-5, 1e-4, 1e-3\}$, Server epochs $M\in\{1, 10, 20\}$.
\item \texttt{Center}: Server learning rate $\eta\in\{1e-5, 1e-4, 1e-3\}$, Server epochs $M\in\{10, 20, 50\}$.
\item \texttt{FedAsync}: Adaptive constant $a\in\{0.1, 0.5, 0.9\}$ and mixing constant $\alpha\in\{0.2, 0.4, 0.8\}$.
\item \texttt{FedBuff}: Server learning rate $\eta\in\{0.01, 0.1, 1\}$ and buffer size $M\in\{10, 20, 50\}$.
\item \texttt{CA2FL}: Server learning rate $\eta\in\{0.01, 0.1, 1\}$ and buffer size $M\in\{10, 20, 50\}$.
\item \texttt{Feddle}: Server learning rate $\eta\in\{1e-5, 1e-4, 1e-3\}$, server epochs $M\in\{1, 10, 20\}$.
\end{itemize}

Model atlas size and fallback regularization strength of \texttt{Feddle} are determined as discussed in \cref{sec:exp:abl}. \texttt{HCFL} does not have additional hyperparameters as it trains a network at the server the same as the clients do.

\paragraph{Experimental Details.}
To account for client delays, we introduce staleness by sampling each client's delay from a half-normal distribution, which matches the practical distribution as observed in previous work~\citep{nguyen2022federated}. Specifically, we take the absolute value of a sample drawn from a zero-mean Gaussian distribution and consider two standard deviations: 5 and 20. To avoid re-sampling clients that have not yet reported their model updates, each client is sampled only once until its update is received in the experiments.
For a fair comparison across methods, each experiment is repeated three times with different random seeds. The random seed controls  network initialization, client sampling order, client delay, and data partitioning. We evaluate the global model every 10 communication rounds and record the maximum value from the last five evaluation rounds as the final model performance.
Finally, we compute the mean and standard deviation of three runs to provide a robust estimate of each method's performance.

\section{Computational Efficiency}
\label{sec:computation}
The optimization of \texttt{Feddle} as described in \cref{eq:kkkk} requires computing gradient with respect to all anchors $\bar \ba_1\ldots \ba_{|A|} \in \mathbb{R}^d$. To enhance scalability, we reduce GPU memory usage and enable distributed computing by applying the following technique. During forward propagation (i.e.\ when computing \cref{eq:kkkk}), we accumulate the anchors to the global model while stopping gradient propagation:
\begin{equation}
    \label{eq:poop}
    \w' = \operatorname{stop\_grad}(\w^k + c_1\bar\ba_1 + \ldots c_{|\A|}\bar \ba_{|\A|}).
\end{equation}
During backpropagation, the gradient of the coefficients is computed as:
\begin{equation}
    \label{eq:mmnn}
    \forall m=1,\ldots, |\A|,\quad \partial \ell/\partial c_m = \langle
\partial\ell/\partial\w', \bar \ba_m\rangle.
\end{equation}
For out-of-domain data availability, we replace $\ell$ with $h$. By employing \cref{eq:mmnn,eq:poop}, the anchors are excluded from the gradient computation graph, allowing us to distribute the gradient calculations across multiple nodes while storing the anchors and model separately. This design enables \texttt{Feddle} to support large models and model atlas. However, as discussed in \cref{sec:exp:abl}, \texttt{Feddle} works well with a moderately sized model atlas, and its computation cost is comparable to fine-tuning a model, which is typically manageable on a server. A comparison of the computation complexity across different approaches is provided in \cref{app:res:computation}.
\begin{table*}[t!]
\centering
\resizebox{\textwidth}{!}{
\begin{tabular}{llcccccccccc}
\toprule[1.5pt]
\multirow{2}{*}{Dataset} & \multirow{2}{*}{Method} & \multirow{2}{*}{ID} & \multicolumn{4}{c}{ResNet18} & \multicolumn{4}{c}{ViT} \\
\cmidrule(lr){4-7} \cmidrule(lr){8-11}
& & & \footnotesize Dir(0.1), $\N(20)$ & \footnotesize Dir(0.1), $\N(5)$ & \footnotesize Dir(0.3), $\N(20)$ & \footnotesize Dir(0.3), $\N(5)$ & \footnotesize Dir(0.1), $\N(20)$ & \footnotesize Dir(0.1), $\N(5)$ & \footnotesize Dir(0.3), $\N(20)$ & \footnotesize Dir(0.3), $\N(5)$ \\
\specialrule{0em}{1pt}{1pt}
\hline
\specialrule{0em}{1pt}{1pt}
\multirow{11}{*}{CIFAR-10} 
& Center & \checkmark & \multicolumn{4}{c}{$77.3\pm0.8$} & \multicolumn{4}{c}{$92.3\pm0.0$} \\
& Fed$+$FT & \checkmark 
  & \underline{$78.4\pm0.6$} 
  & \underline{$81.3\pm0.2$} 
  & \underline{$80.1\pm0.2$} 
  & \underline{$82.2\pm0.2$} 
  & \underline{$96.5\pm0.1$} 
  & \underline{\bftab 97.1 $\pm$ 0.0} 
  & \underline{$96.2\pm0.2$} 
  & \underline{$97.0\pm0.0$} \\
& HFCL & \checkmark 
  & \underline{$80.2\pm0.9$} 
  & \underline{$83.7\pm1.0$} 
  & \underline{$82.1\pm0.3$} 
  & \underline{$85.7\pm0.4$} 
  & \underline{$96.1\pm0.1$} 
  & \underline{$96.9\pm0.1$} 
  & \underline{$96.4\pm0.0$} 
  & \underline{$96.9\pm0.0$} \\
& FedDF-ID & \checkmark 
  & $57.3\pm1.5$ 
  & $75.4\pm0.4$ 
  & $74.4\pm1.0$ 
  & \underline{$82.1\pm0.3$} 
  & $91.1\pm0.7$ 
  & \underline{$96.8\pm0.2$} 
  & \underline{$93.5\pm0.2$} 
  & \underline{$96.8\pm0.0$} \\
& Ours-ID & \checkmark 
  & \underline{\bftab 86.3 $\pm$ 0.4} 
  & \underline{\bftab 86.7 $\pm$ 0.5} 
  & \underline{\bftab 87.3 $\pm$ 0.4} 
  & \underline{\bftab 87.8 $\pm$ 0.3} 
  & \underline{\bftab 97.1 $\pm$ 0.0} 
  & \underline{\bftab 97.1 $\pm$ 0.0} 
  & \underline{\bftab 97.7 $\pm$ 0.1} 
  & \underline{\bftab 97.5 $\pm$ 0.1} \\
\arrayrulecolor{lightgray}
\cmidrule(lr){2-11}
\arrayrulecolor{black}
& FedAvg 
  & 
  & $56.4\pm5.3$ 
  & $75.1\pm1.0$ 
  & $72.4\pm1.0$ 
  & \underline{$85.2\pm0.2$} 
  & $87.0\pm1.5$ 
  & \underline{$94.6\pm1.1$} 
  & $89.4\pm0.1$ 
  & \underline{$95.5\pm0.2$} \\
& FedAsync 
  & 
  & $65.1\pm3.5$ 
  & $74.9\pm2.4$ 
  & \underline{$79.3\pm1.9$} 
  & \underline{$83.9\pm1.4$} 
  & $89.4\pm1.3$ 
  & \underline{$94.7\pm0.8$} 
  & \underline{$92.8\pm0.7$} 
  & \underline{$96.3\pm0.3$} \\
& FedBuff 
  & 
  & $59.6\pm3.3$ 
  & $72.8\pm7.3$ 
  & $69.3\pm8.0$ 
  & \underline{$77.9\pm4.3$} 
  & \underline{$96.2\pm0.1$} 
  & \underline{$96.1\pm0.3$} 
  & \underline{$97.0\pm0.2$} 
  & \underline{$96.9\pm0.1$} \\
& CA2FL 
  & 
  & $64.4\pm7.2$ 
  & \underline{$80.2\pm0.9$} 
  & $71.6\pm5.7$ 
  & $76.3\pm6.8$ 
  & \underline{$96.5\pm0.1$} 
  & \underline{$96.1\pm0.1$} 
  & \underline{$97.1\pm0.1$} 
  & \underline{$96.9\pm0.0$} \\
& FedDF-OOD 
  & 
  & $29.7\pm1.1$ 
  & $29.2\pm5.6$ 
  & $42.9\pm2.9$ 
  & $42.8\pm2.1$ 
  & $29.7\pm1.1$ 
  & $29.2\pm5.6$ 
  & $42.9\pm2.9$ 
  & $42.8\pm2.1$ \\
& Ours-OOD 
  & 
  & \underline{\bftab 82.2 $\pm$ 1.4} 
  & \underline{\bftab 83.1 $\pm$ 0.3} 
  & \underline{\bftab 86.1 $\pm$ 0.5} 
  & \underline{\bftab 88.3 $\pm$ 0.6} 
  & \underline{\bftab 97.0 $\pm$ 0.2} 
  & \underline{\bftab 96.7 $\pm$ 0.4} 
  & \underline{\bftab 97.5 $\pm$ 0.1} 
  & \underline{\bftab 97.5 $\pm$ 0.0} \\
\bottomrule[1.5pt]
\end{tabular}
}
\caption{\textbf{Comparisons of different approaches} on CIFAR10. These experiments consider two data heterogeneity levels (Dir(0.1), Dir(0.3)) and two delay levels ($\N(5), \N(20)$). ``ID" indicates whether the approach uses in-domain data. If so, 1000 samples are provided. Performance higher than \texttt{Center} is underlined. The best performance in both data availabilities is highlighted by bold.}
\label{tab:results-cifar10}
\end{table*}
\begin{table*}[t]
\centering
\resizebox{0.65\textwidth}{!}{
\begin{tabular}{llccccc}
\toprule[1.5pt]
\multirow{2}[1]{*}{Dataset} & \multirow{2}[1]{*}{Method} & \multirow{2}[1]{*}{ID} & \multicolumn{4}{c}{CNN} \\
\cmidrule(lr){4-7}
& & & \footnotesize D(0.1), $\N(20)$ & \footnotesize D(0.1), $\N(5)$ & \footnotesize D(0.3), $\N(20)$ & \footnotesize D(0.3), $\N(5)$ \\
\specialrule{0em}{1pt}{1pt}
\hline
\specialrule{0em}{1pt}{1pt}
\multirow{11}[1]{*}{CIFAR-10} 
  & Center & \tikzcmark 
      & \multicolumn{4}{c}{$42.1\pm 0.1$} \\
  & Fed$+$FT  & \tikzcmark 
      & \underline{$44.3 \pm 0.5$} 
      & \underline{$47.4 \pm 0.2$} 
      & \underline{$46.9 \pm 0.6$} 
      & \underline{$51.3 \pm 0.6$} \\
  & HFCL  & \tikzcmark 
      & \underline{$44.7 \pm 0.5$} 
      & \underline{$47.1 \pm 0.4$} 
      & \underline{$47.6 \pm 0.1$} 
      & \underline{$51.2 \pm 1.1$} \\
  & FedDF-ID & \tikzcmark 
      & $18.0 \pm 1.0$ 
      & $19.8 \pm 2.0$ 
      & $37.3 \pm 1.2$ 
      & $39.5 \pm 1.3$ \\
  & Ours-ID  & \tikzcmark 
      & \underline{\bftab 51.2 $\pm$ 1.4} 
      & \underline{\bftab 52.9 $\pm$ 0.3} 
      & \underline{\bftab 51.5 $\pm$ 1.1} 
      & \underline{\bftab 53.3 $\pm$ 0.6} \\
  \arrayrulecolor{lightgray}
  \cmidrule(lr){4-7}
  \arrayrulecolor{black}
  & FedAvg & \tikzxmark 
      & $31.0 \pm 1.3$ 
      & $35.0 \pm 3.8$ 
      & $37.2 \pm 1.2$ 
      & \underline{$52.0 \pm 2.6$} \\
  & FedAsync & \tikzxmark 
      & $35.5 \pm 1.4$ 
      & $40.4 \pm 0.2$ 
      & \underline{$43.2 \pm 0.8$} 
      & \underline{$45.7 \pm 3.1$} \\
  & FedBuff & \tikzxmark 
      & $26.8 \pm 3.3$ 
      & $33.2 \pm 7.6$ 
      & $32.6 \pm 4.1$ 
      & \underline{$44.0 \pm 7.2$} \\
  & CA2FL & \tikzxmark 
      & $27.6 \pm 6.9$ 
      & $36.6 \pm 2.7$ 
      & $37.1 \pm 3.4$ 
      & \underline{$43.9 \pm 4.5$} \\
  & FedDF-OOD & \tikzxmark 
      & $18.5 \pm 1.9$ 
      & $19.4 \pm 1.9$ 
      & $35.1 \pm 0.2$ 
      & $39.3 \pm 1.4$ \\
  & Ours-OOD & \tikzxmark 
      & \bftab 37.6 $\pm$ 4.5 
      & \underline{\bftab 42.3 $\pm$ 1.6} 
      & \underline{\bftab 44.5 $\pm$ 0.8} 
      & \underline{\bftab 53.1 $\pm$ 2.5} \\
      
\midrule
\multirow{11}[1]{*}{CIFAR-100} 
  & Center & \tikzcmark 
      & \multicolumn{4}{c}{$23.3\pm0.5$} \\
  & Fed$+$FT  & \tikzcmark 
      & \underline{$26.2 \pm 0.7$} 
      & \underline{$29.7 \pm 0.3$} 
      & \underline{$27.0 \pm 0.3$} 
      & \underline{$30.3 \pm 0.3$} \\
  & HFCL  & \tikzcmark 
      & \underline{$28.9 \pm 0.4$} 
      & \underline{$30.8 \pm 0.3$} 
      & \underline{$29.3 \pm 0.4$} 
      & \underline{$31.9 \pm 0.5$} \\
  & FedDF-ID & \tikzcmark 
      & \underline{$28.7 \pm 0.5$} 
      & \underline{$29.3 \pm 0.6$} 
      & \underline{$28.9 \pm 0.1$} 
      & \underline{$29.1 \pm 0.2$} \\
  & Ours-ID  & \tikzcmark 
      & \underline{\bftab 32.7 $\pm$ 0.6} 
      & \underline{\bftab 38.6 $\pm$ 1.6} 
      & \underline{\bftab 38.0 $\pm$ 0.7} 
      & \underline{\bftab 40.3 $\pm$ 0.8} \\
  \arrayrulecolor{lightgray}
  \cmidrule(lr){4-7}
  \arrayrulecolor{black}
  & FedAvg & \tikzxmark 
      & $21.7 \pm 2.2$ 
      & \underline{$27.7 \pm 1.9$} 
      & \underline{$25.9 \pm 1.2$} 
      & \underline{$32.1 \pm 0.7$} \\
  & FedAsync & \tikzxmark 
      & \underline{$24.3 \pm 1.2$} 
      & \underline{$30.8 \pm 0.7$} 
      & \underline{$27.0 \pm 1.1$} 
      & \underline{$33.0 \pm 0.4$} \\
  & FedBuff & \tikzxmark 
      & $22.5 \pm 2.2$ 
      & \underline{$31.5 \pm 2.4$} 
      & \underline{$24.9 \pm 4.4$} 
      & \underline{$31.1 \pm 2.5$} \\
  & CA2FL & \tikzxmark 
      & $22.3 \pm 0.7$ 
      & \underline{$29.9 \pm 3.8$} 
      & \underline{$25.2 \pm 3.2$} 
      & \underline{$32.0 \pm 2.8$} \\
  & FedDF-OOD & \tikzxmark 
      & \underline{$24.1 \pm 0.4$} 
      & \underline{$29.8 \pm 0.5$} 
      & \underline{$27.1 \pm 0.4$} 
      & \underline{$33.3 \pm 0.9$} \\
  & Ours-OOD & \tikzxmark 
      & \underline{\bftab 31.9 $\pm$ 1.8} 
      & \underline{\bftab 37.2 $\pm$ 1.7} 
      & \underline{\bftab 36.9 $\pm$ 1.4} 
      & \underline{\bftab 40.4 $\pm$ 0.6} \\
      
\midrule
\multirow{11}[1]{*}{Fashion-MNIST} 
  & Center & \tikzcmark 
      & \multicolumn{4}{c}{$82.6\pm0.9$} \\
  & Fed$+$FT  & \tikzcmark 
      & \underline{$84.5 \pm 0.2$} 
      & \underline{$86.1 \pm 0.2$} 
      & \underline{$85.1 \pm 0.0$} 
      & \underline{$87.1 \pm 0.0$} \\
  & HFCL  & \tikzcmark 
      & \underline{$84.4 \pm 0.4$} 
      & \underline{$86.2 \pm 0.1$} 
      & \underline{$85.2 \pm 0.3$} 
      & \underline{$87.1 \pm 0.2$} \\
  & FedDF-ID & \tikzcmark 
      & $48.9 \pm 5.8$ 
      & $44.0 \pm 5.0$ 
      & $72.9 \pm 0.9$ 
      & $71.7 \pm 4.9$ \\
  & Ours-ID  & \tikzcmark 
      & \underline{\bftab 86.5 $\pm$ 0.2} 
      & \underline{\bftab 87.3 $\pm$ 0.0} 
      & \underline{\bftab 86.4 $\pm$ 0.2} 
      & \underline{\bftab 87.1 $\pm$ 0.1} \\
  \arrayrulecolor{lightgray}
  \cmidrule(lr){4-7}
  \arrayrulecolor{black}
  & FedAvg & \tikzxmark 
      & $60.8 \pm 4.4$ 
      & $80.7 \pm 1.0$ 
      & $81.1 \pm 0.7$ 
      & \underline{$86.6 \pm 0.1$} \\
  & FedAsync & \tikzxmark 
      & $78.4 \pm 2.2$ 
      & $80.9 \pm 1.8$ 
      & \underline{$83.4 \pm 0.6$} 
      & \underline{$85.4 \pm 1.0$} \\
  & FedBuff & \tikzxmark 
      & $79.4 \pm 1.0$ 
      & $82.4 \pm 1.5$ 
      & $81.6 \pm 4.6$ 
      & \underline{$85.4 \pm 1.0$} \\
  & CA2FL & \tikzxmark 
      & $78.9 \pm 0.9$ 
      & \underline{$85.4 \pm 0.7$} 
      & $82.1 \pm 2.7$ 
      & \underline{$85.5 \pm 0.6$} \\
  & FedDF-OOD & \tikzxmark 
      & $47.4 \pm 6.8$ 
      & $54.7 \pm 5.3$ 
      & $71.9 \pm 0.4$ 
      & $79.1 \pm 2.6$ \\
  & Ours-OOD & \tikzxmark 
      & \bftab 81.4 $\pm$ 2.4 
      & \underline{\bftab 86.1 $\pm$ 2.0} 
      & \underline{\bftab 85.1 $\pm$ 0.8} 
      & \underline{\bftab 88.2 $\pm$ 1.3} \\
      
\bottomrule[1.5pt]
\end{tabular}  
}
\caption{\textbf{Comparisons of different approaches using CNN}. These experiments consider two data heterogeneity levels (Dir(0.1), Dir(0.3)) and two delay levels ($\N(5), \N(20)$). ``ID" indicates whether the approach uses in-domain data. If so, 1000 samples are provided. Performance higher than \texttt{Center} is underlined. The best performance in both data availabilities is highlighted by bold.}
\label{tab:results-out-domain-cnn}
\end{table*}
\section{Additional Results}
\label{additional-results}
In this section, we present additional experimental results. In \cref{app:res:cifar10} we report results on CIFAR100. In \cref{app:res:cnn}, we show outcomes for a CNN trained from scratch. \cref{app:res:computation} compares the computation cost across different approaches, while \cref{app:res:pattern} illustrates additional search patterns of \texttt{Feddle}. Finally, \cref{app:res:loss} provides further analysis of the optimization signal from the surrogate loss applied in \texttt{Feddle}.

\subsection{Experiments on CIFAR10}
\label{app:res:cifar10}
\Cref{tab:results-cifar10} summarizes the results on CIFAR10. The experimental settings align with those in \cref{tab:results-resnet-vit}. As observed previously, \texttt{Feddle} consistently outperforms the baseline methods and is less affected by high data heterogeneity and communication delays.

\subsection{Training from Scratch on CNN}
\label{app:res:cnn}
We supplement our results with experiments training a CNN from scratch. As shown in \cref{tab:results-out-domain-cnn}, \texttt{Feddle} consistently outperforms all baselines. In contrast to the results with pretrained ResNet and ViT (see \cref{tab:results-resnet-vit,tab:results-cifar10}), we observe that \texttt{Feddle} may not outperform \texttt{Center} in the OOD setting under the most challenging scenario with Dir(0.1) and $\N(20)$. However, in the OOD setting \texttt{Feddle} often significantly outperforms the best baseline and its performance is less impacted by heterogeneity and delay. Moreover, when training a randomly initialized CNN, \texttt{Feddle} can leverage ImageNet data to guide the client models trained on Fashion-MNIST, further demonstrating the robustness of our framework in the OOD scenarios. 

\subsection{Server-Side Computation Cost}
\label{app:res:computation}
To measure the computation cost, we construct a federated learning group with 1000 clients and sample 50 clients per round. For simplicity, and to avoid the dynamics of asynchronous communication, we assume that all 50 clients report on time. We use the ViT network and fine-tune it with LoRA as described in \cref{app:setting}. For \texttt{Feddle}, the model atlas size is set to twice the number of clients sampled per round (i.e.\ atlas size = 100), which is good for performance as discussed in~\cref{sec:exp:abl}. We also set the buffer size of \texttt{FedBuff} to 100. For methods that conduct training at the server, we fix the number of iterations to be the same. In this work, we search for the optimal client training iterations based on \texttt{FedAvg} and apply the same setting to all methods (see \cref{app:setting}), ensuring that client-side computations remain identical. Therefore, our comparison focuses on the server-side computation cost. Specifically, we report two metrics that vary significantly across approaches: \textbf{1)} GFLOPs, which indicate the amount of computation performed by the server, and \textbf{2)} Cache size, which reflects the amount of intermediate results cached by each approach (note that all approaches cache the global model at a minimum).

As shown in the GFLOPs column of \cref{tab:metrics_scientific}, \texttt{FedAvg}, \texttt{FedAsync} and \texttt{FedBuff} incur almost negligible GFLOPs at the server since they simply aggregate the model updates. The computation cost for \texttt{Fed+FT} and \texttt{HFCL} represents the cost of fine-tuning the model at the server. In comparison, \texttt{Feddle} requires approximately 15\% more GFLOPs on the server because the gradient must be projected onto the anchors (see \cref{eq:mmnn}).  In contrast, \texttt{FedDF} demands over 10$\times$ more computation due to the inference perfomred on all reported client models for knowledge distillation.

Additionally, server must cache intermediate results during training, such as the global model. The cache size varies significantly among the approaches. As shown in the Cache (MB) column of \cref{tab:metrics_scientific}, \texttt{Fed+FT}, \texttt{HFCL}, \texttt{FedDF}, \texttt{FedAvg} and \texttt{FedAsync} have the smallest cache size, corresponding to maintaining only the global model on the server. For these methods, all received model updates can be discarded after aggregation or training. In contrast, \texttt{Feddle} and \texttt{FedBuff} has a 35\% larger cache size, which corresponds to storing 50 LoRAs saved in the model atlas or buffer. Notably, due to the cached update calibration incorporated in \texttt{CA2FL}, a vector of the same size as the full network is saved for each client. Consequently, the cache size of \texttt{CA2FL} is three orders of magnitudes larger for a total of 1000 clients, and it scales with the number of clients.

In conclusion, \texttt{Feddle} requires server computation similar to that of fine-tuning a model. Although it caches multiple model updates depending on the size of the model atlas, this does not significantly increase the overall cache size. While \texttt{Feddle} is not the most lightweight framework among the approaches, its computation cost is generally manageable for a server, and it achieves the best performance. Moreover, due to its constrained size of the model atlas, \texttt{Feddle}'s computation complexity and cache size does not rapidly scale up with the size of the federated learning group.

\begin{table}[t!]
\centering
\resizebox{0.65\linewidth}{!}{%
\begin{tabular}{lcc}
\toprule[1.5pt]
Method    & GFLOPs           & Cache (MB)     \\
\specialrule{0em}{1pt}{1pt}
\hline
\specialrule{0em}{1pt}{1pt}
Fed+FT    & \(7.1\times10^{5}\)   & \(3.4\times10^{2}\)   \\
HFCL      & \(7.1\times10^{5}\)   & \(3.4\times10^{2}\)   \\
FedDF     & \textcolor{wacvblue}{\(1.8\times10^{7}\)}   & \(3.4\times10^{2}\)   \\
FedAvg      & \(1.8\times10^{1}\)   & \(3.4\times10^{2}\)   \\
FedAsync      & \(1.8\times10^{1}\)   & \(3.4\times10^{2}\)   \\
FedBuff   & \(1.8\times10^{1}\)   & \(4.6\times10^{2}\)   \\
CA2FL     & \(1.9\times10^{1}\)   & \textcolor{wacvblue}{\(3.4\times10^{5}\)}   \\
Feddle (Ours)    & \(8.2\times10^{5}\)   & \(4.6\times10^{2}\)   \\
\bottomrule[1.5pt]
\end{tabular}%
}
\caption{\textbf{Comparison of the server-side computation cost across approaches.} The cost of \texttt{Center} is similar as \texttt{HFCL} and \texttt{Fed+FT}. The highest cost is marked in blue.}
\label{tab:metrics_scientific}
\end{table}

\begin{figure*}[t]
    \centering
    \begin{subfigure}[t]{.25\textwidth}
        \centering
        \includegraphics[width=\linewidth]{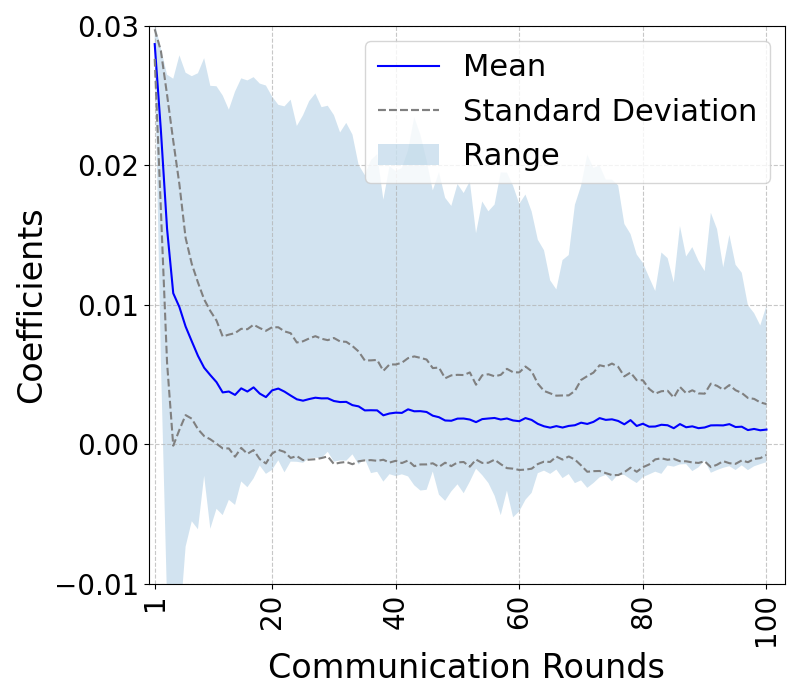}
        \caption{Dir(0.1), $\N(5)$}
    \end{subfigure}%
    \begin{subfigure}[t]{.25\textwidth}
        \centering
        \includegraphics[width=\linewidth]{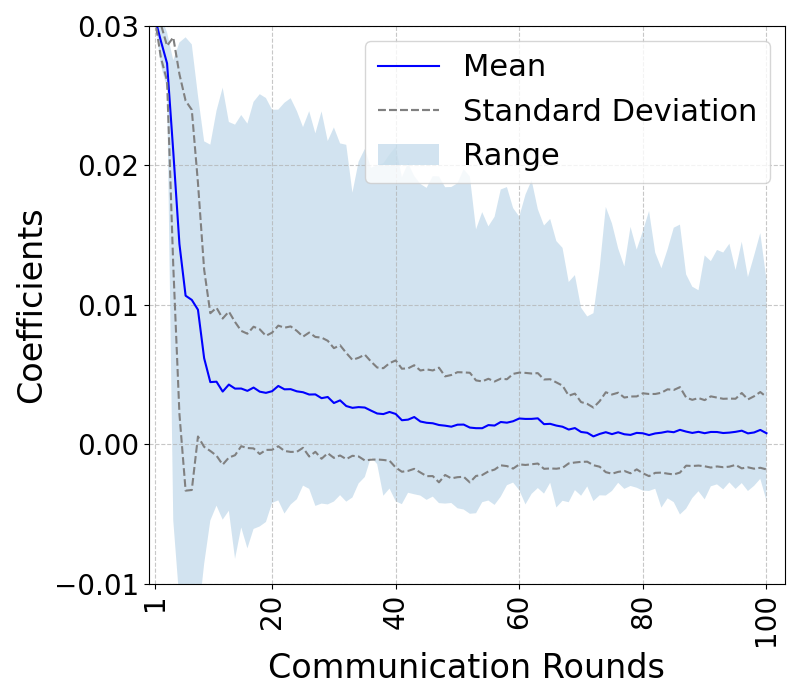}
        \caption{Dir(0.3), $\N(20)$}
    \end{subfigure}%
    \begin{subfigure}[t]{.25\textwidth}
        \centering
        \includegraphics[width=\linewidth]{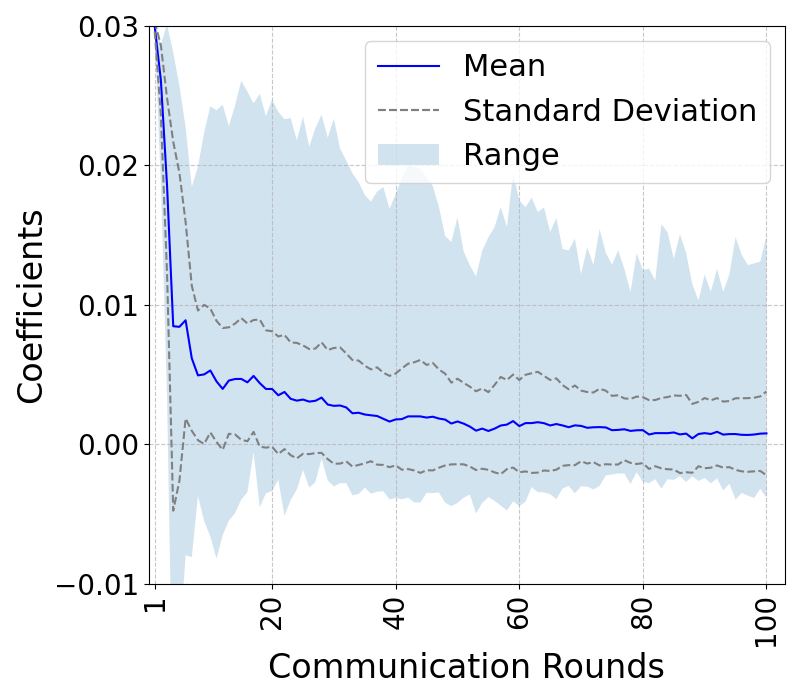}
        \caption{Dir(0.3), $\N(5)$}
    \end{subfigure}
\caption{\textbf{Statistics of coefficients identified by \texttt{Feddle} for in-domain (ID) data availability} using ResNet18 and CIFAR-10.}
    \label{fig:search_pattern}
\end{figure*}

\begin{figure*}[t]
    \centering
    \begin{subfigure}[t]{.25\textwidth}
        \centering
        \includegraphics[width=\linewidth]{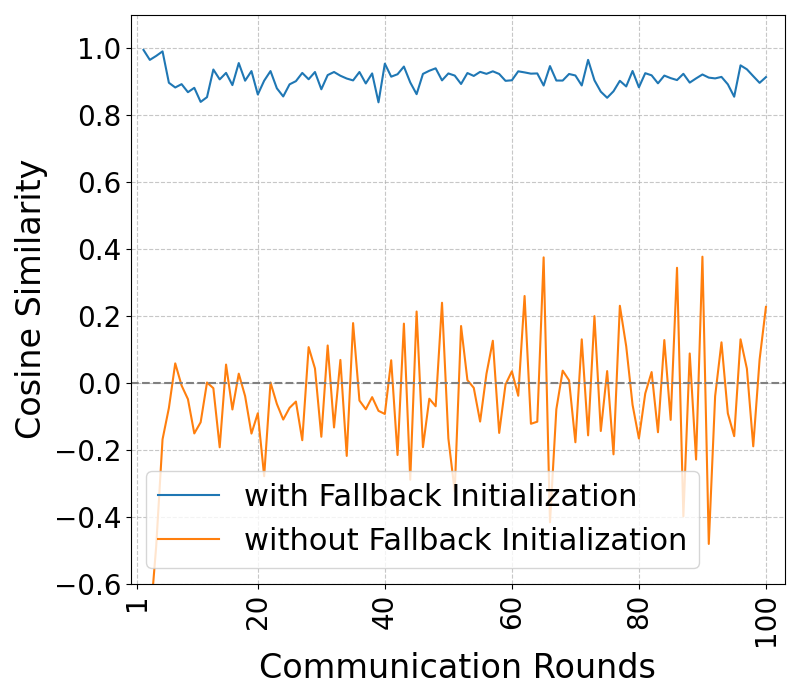}
        \caption{Dir(0.1), $\N(5)$}
    \end{subfigure}%
    \begin{subfigure}[t]{.25\textwidth}
        \centering
        \includegraphics[width=\linewidth]{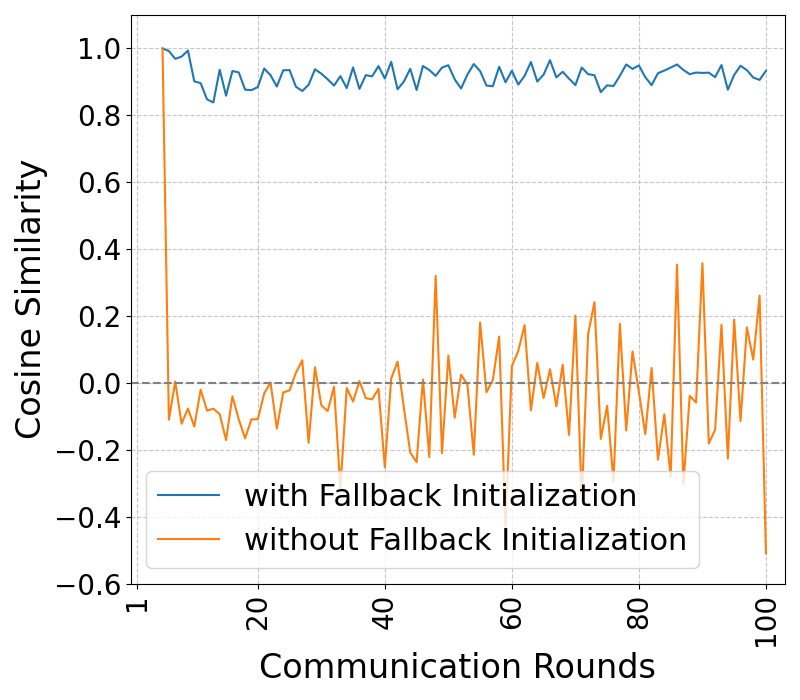}
        \caption{Dir(0.3), $\N(20)$}
    \end{subfigure}%
    \begin{subfigure}[t]{.25\textwidth}
        \centering
        \includegraphics[width=\linewidth]{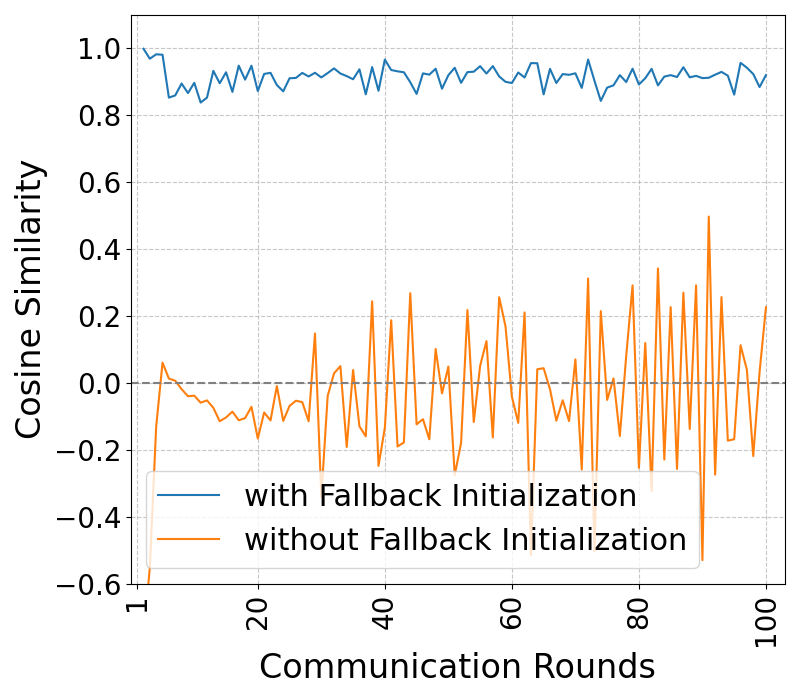}
        \caption{Dir(0.3), $\N(5)$}
    \end{subfigure}
\caption{\textbf{Similarity of the coefficients' optimization direction between in-domain (ID) and out-of-domain (OOD) data} using ResNet18 and CIFAR-10.}
    \label{fig:ood_signal}
\end{figure*}
\begin{figure*}[t!]
    \centering
    \begin{subfigure}[t]{.25\textwidth}
        \centering
        \includegraphics[width=\linewidth]{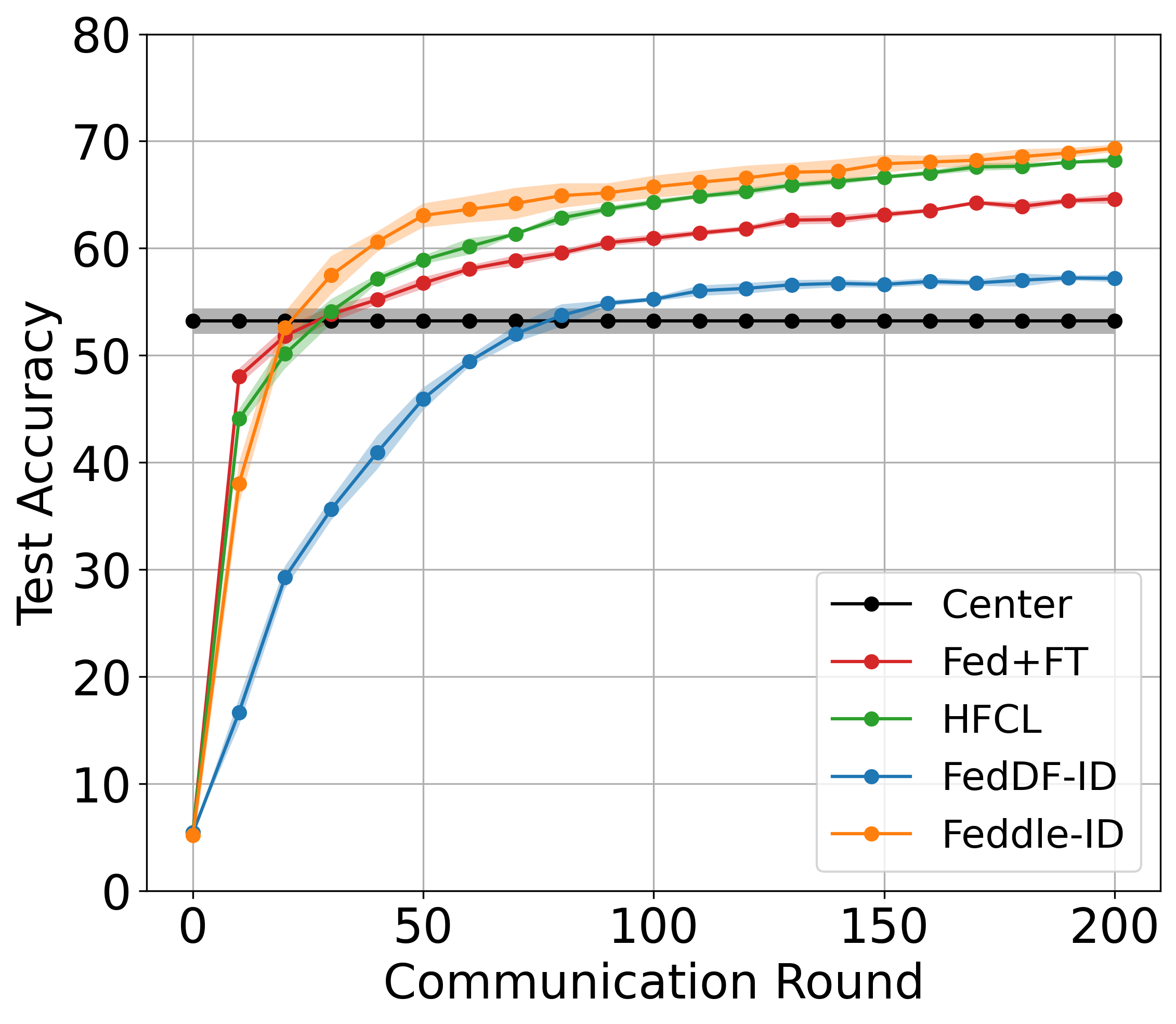}
        \caption{Dir(0.1), $\N(5)$}
    \end{subfigure}%
    \begin{subfigure}[t]{.25\textwidth}
        \centering
        \includegraphics[width=\linewidth]{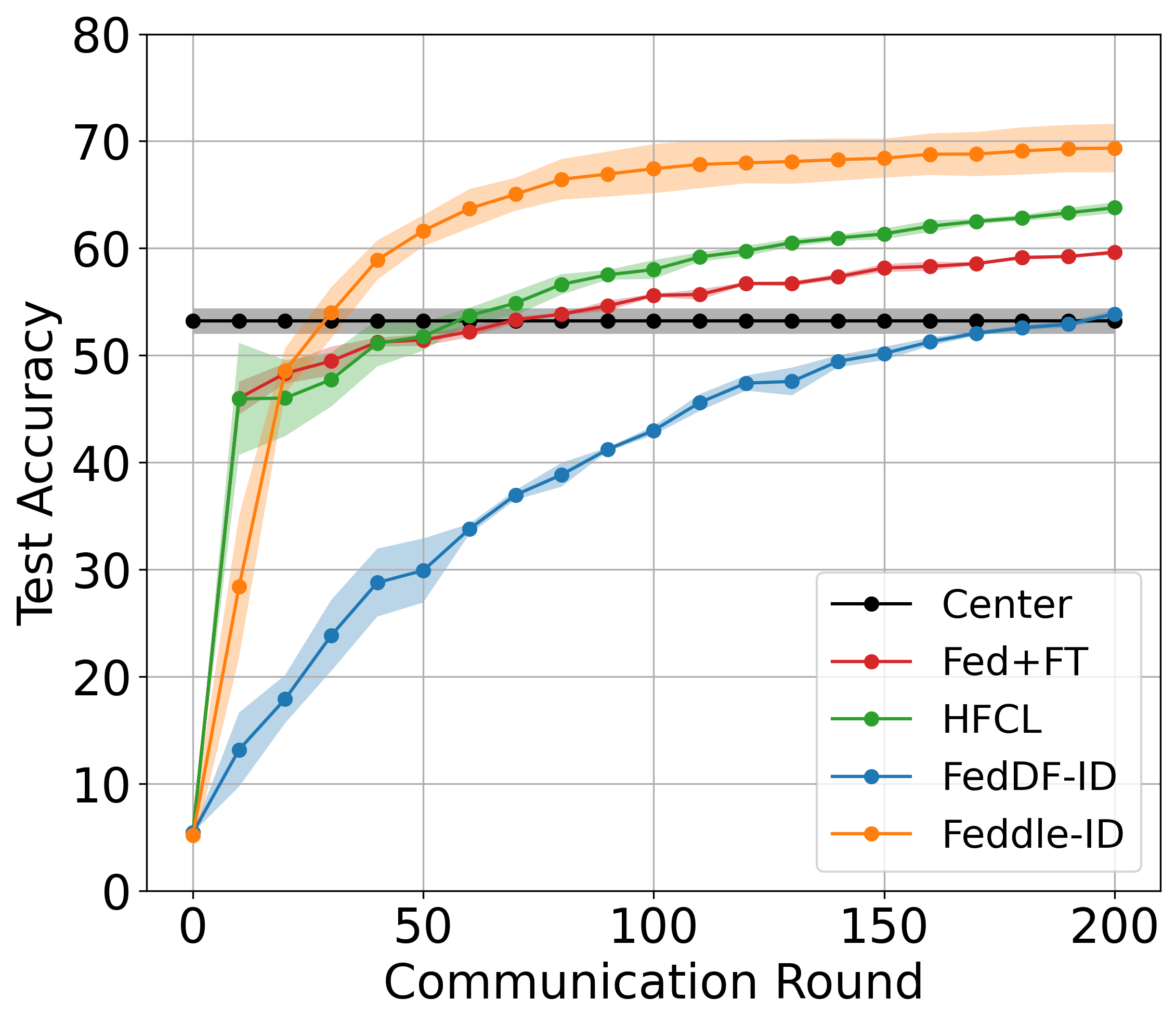}
        \caption{Dir(0.3), $\N(20)$}
    \end{subfigure}%
    \begin{subfigure}[t]{.25\textwidth}
        \centering
        \includegraphics[width=\linewidth]{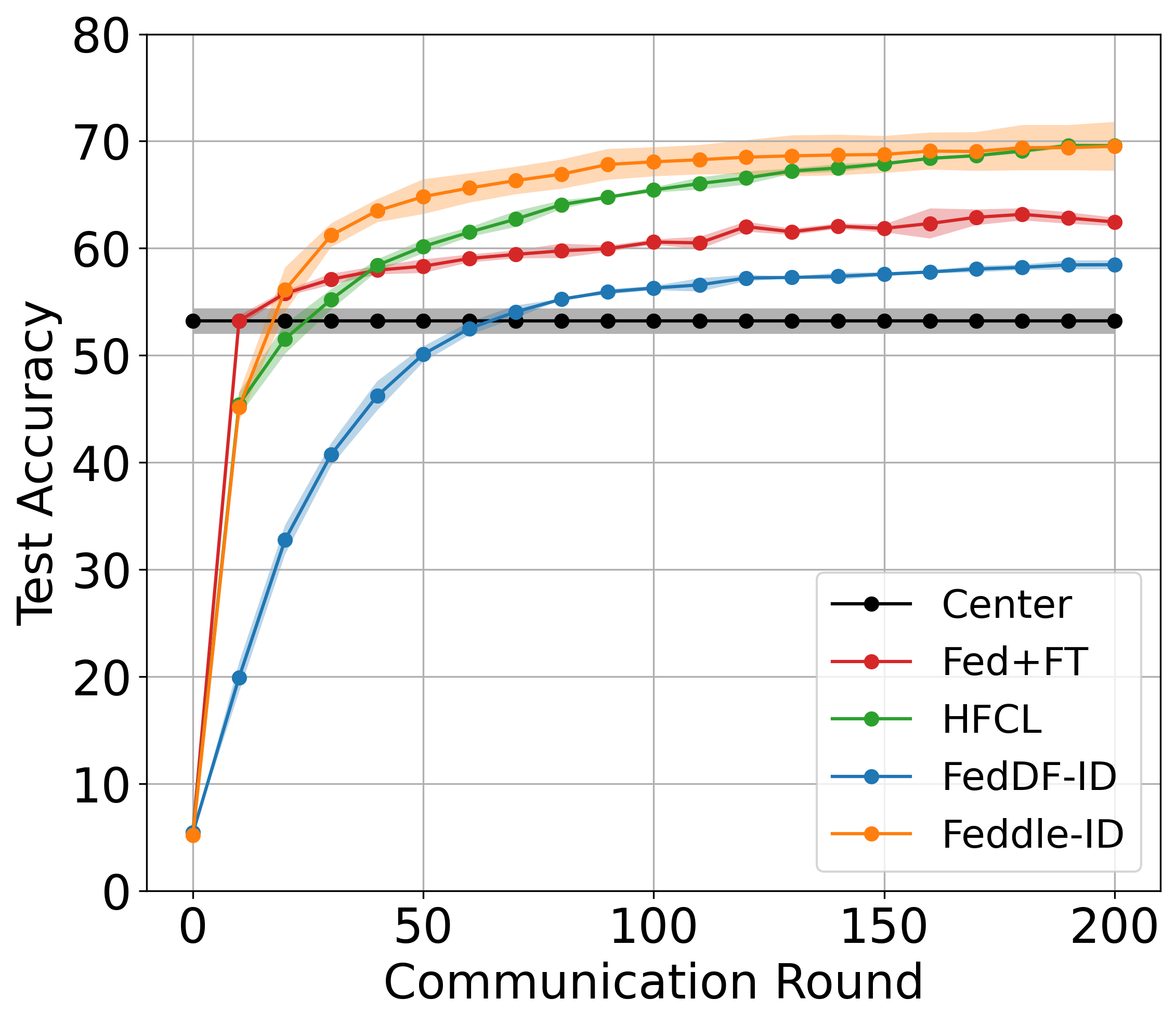}
        \caption{Dir(0.3), $\N(5)$}
    \end{subfigure}
\caption{\textbf{Convergence plots} for ResNet18 on CIFAR100 under OOD data availability.}
    \label{fig:cifar100_convergence_plot_id}
\end{figure*}

\begin{figure*}[t!]
    \centering
    \begin{subfigure}[t]{.25\textwidth}
        \centering
        \includegraphics[width=\linewidth]{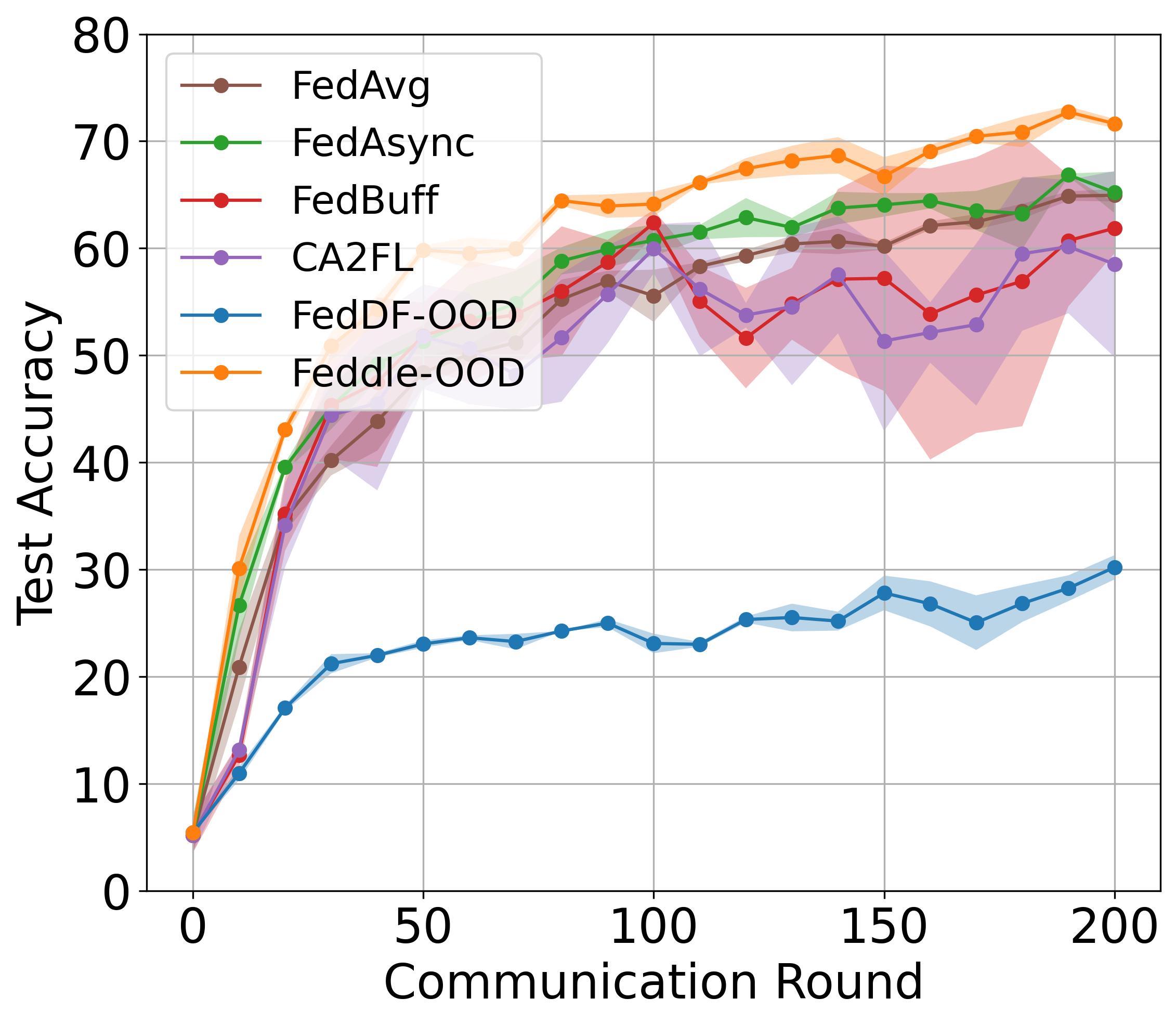}
        \caption{Dir(0.1), $\N(5)$}
    \end{subfigure}%
    \begin{subfigure}[t]{.25\textwidth}
        \centering
        \includegraphics[width=\linewidth]{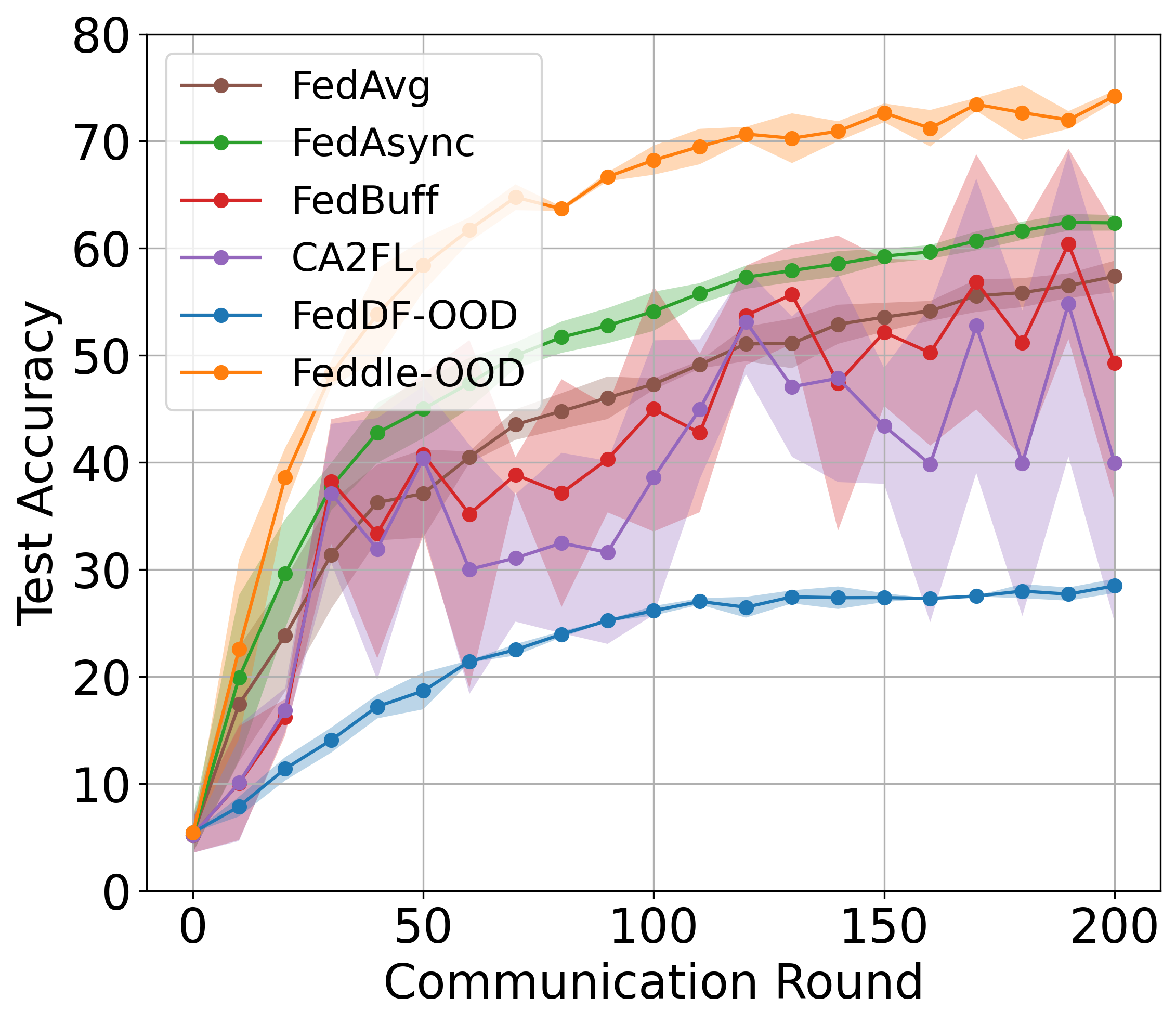}
        \caption{Dir(0.3), $\N(20)$}
    \end{subfigure}%
    \begin{subfigure}[t]{.25\textwidth}
        \centering
        \includegraphics[width=\linewidth]{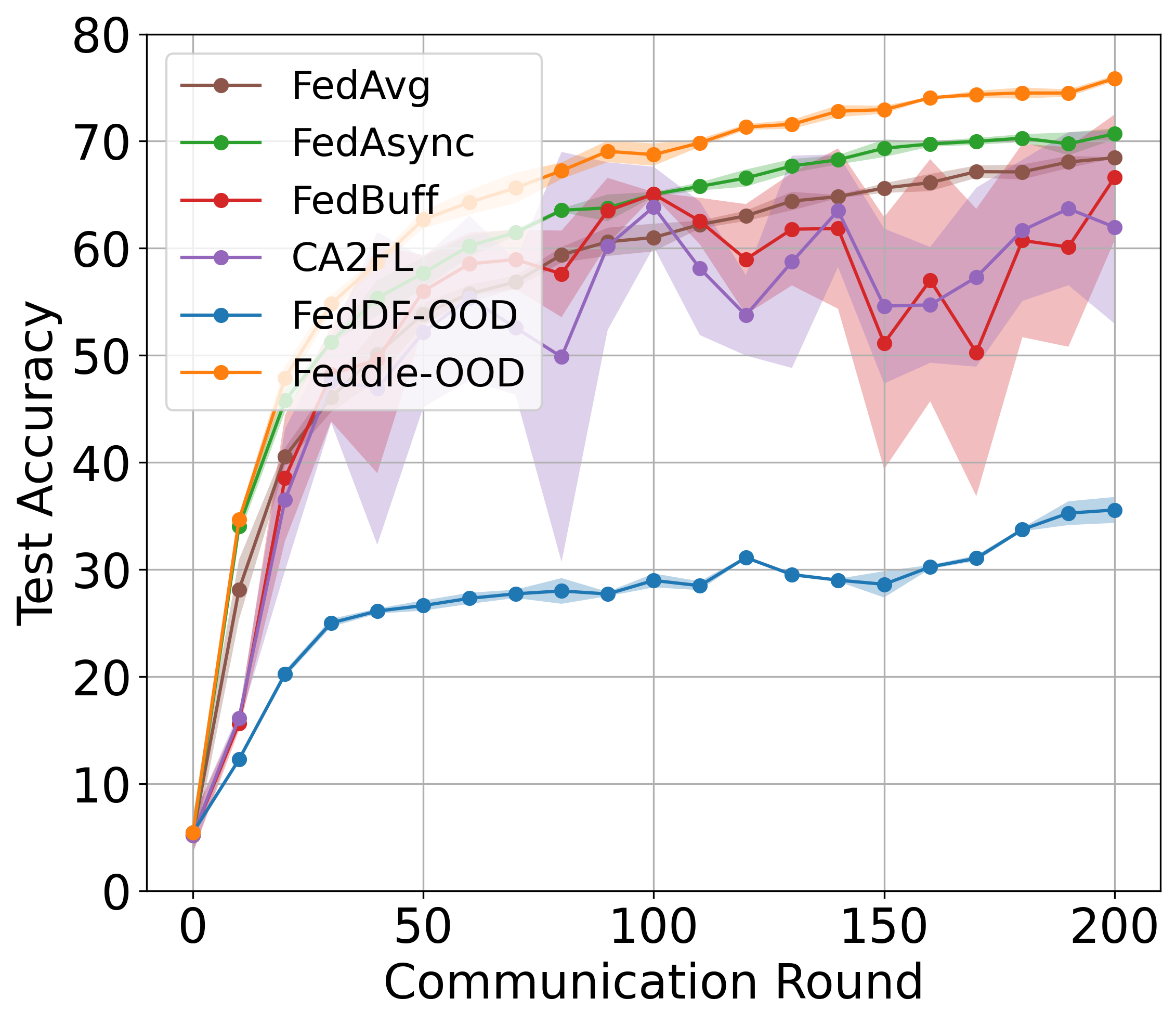}
        \caption{Dir(0.3), $\N(5)$}
    \end{subfigure}
\caption{\textbf{Convergence plots} for ResNet18 on CIFAR100 under ID data availability.}
    \label{fig:cifar100_convergence_plot_out}
\end{figure*}
\subsection{Searched Coefficient Pattern}
\label{app:res:pattern}
In \cref{sec:bg}, we show that model updates reported by clients often contain conflicting information due to data heterogeneity and delayed responses from asynchronous communication. Consequently, the optimal aggregation coefficients computed on the server are not always positive. Since \texttt{Feddle} consistently achieves the best performance by determining the aggregation coefficients under data guidance, we further demonstrate that its searched coefficients indeed include negative values. As illustrated in \cref{fig:search_pattern} (another subplot is provided in \cref{fig:search_patter_dir01_n20}), some of the coefficients deemed optimal by the server are negative. This phenomenon persists across all communication rounds and under different configurations of data heterogeneity and communication delays. This observation highlights the persistent disagreement among clients' model updates and attests to the capability of \texttt{Feddle}.

\subsection{Optimization Signal of the Surrogate Loss}
\label{app:res:loss}
In \cref{sec:method:obj}, we discuss that if the server guides decentralized training using OOD data $\D_S'$ with a surrogate loss function $h$, the induced gradient must satisfy \cref{eq:adfd} to ensure that the optimization direction aligns with the ID server-side data $\D_S$ and the client loss function $\ell$. For convenience, we restate this condition: \begin{equation} \langle\partial h(\D_S',, \cdot)/\partial \boldsymbol c,; \partial \ell(\D,, \cdot)/\partial \boldsymbol c\rangle > 0. \end{equation}
This condition is also incorporated into our theoretical analysis to justify the convergence of \texttt{Feddle} (see \cref{sec:method:analysis}).
We find that this condition is met under various settings in our experiments, and fallback initialization is crucial for its satisfaction. As shown in \cref{fig:ood_signal} (with an additional subplot in \cref{fig:ood_signal_dir01_n20}), without fallback initialization the cosine similarity between the optimization directions derived from ID and OOD data appears random. In contrast, with fallback initialization, these optimization directions become highly aligned, with the cosine similarity approaching 1. Our ablation study (see \cref{sec:exp:abl}) further confirms that without the fallback mechanism, \texttt{Feddle}'s performance deteriorates to random guessing, underscoring the critical role of fallback initialization for applying \texttt{Feddle} in scenarios where only OOD data is available.

Based on these results, we hypothesize that the optimization landscape constructed by the surrogate loss using OOD data may not be globally aligned with that formed by the client loss function using ID data. However, when the starting point is initialized near the optimum for ID data (as fallback can leverage an existing successful framework), the optimization direction derived from OOD data can still align with that of ID data.

\subsection{Convergence Plots}
\label{app:res:convergence}
We present the convergence plots for ResNet18 on CIFAR100 in \cref{fig:cifar100_convergence_plot_id,fig:cifar100_convergence_plot_out}. The plots corresponding to the configuration with Dir(0.1) and $\N(20)$ are shown in \cref{fig:converge-in-domain}.

\subsection{Ablation Studies}
\label{sec:exp:abl}
\noindent\textbf{Fallback Regularization Strength.}
One hyperparameter introduced by \texttt{Feddle} is the regularization strength $\lambda$ of the fallback mechanism. As shown in \cref{fig:ablation-regularization}, with ID data, \texttt{Feddle} appears insensitive to this hyperparameter, while a strong regularization (such as $\lambda=0.1$) may not be beneficial under the OOD data availability. In this work, we adopt 0.0 for ID settings and 0.01 for OOD settings across datasets and models.

\noindent\textbf{Model Atlas Size.}  
The other hyperparameter introduced by \texttt{Feddle} is the atlas size $M$. As shown in \cref{fig:ablation-atlas-size}, expanding the model atlas generally benefits ID settings. However, an excessively large atlas, such as one that more than twice the number of clients sampled per round, can hinder performance in the OOD settings. This is likely because the optimization signal with OOD data is less reliable than ID data, making \texttt{Feddle} more prone to converging to a location that deviates from the original objective in an expanded optimization space. In this work, we always set the atlas size to twice the number of sampled clients, which results in good performance across datasets, models and settings. 

\noindent\textbf{Computation Cost.}
We compare computation cost across different approaches in terms of the server computation complexity and cache size in \cref{app:res:computation}.

\begin{figure}[t!]
\centering
\begin{minipage}[t]{0.48\columnwidth}
    \centering
    \includegraphics[width=\linewidth]{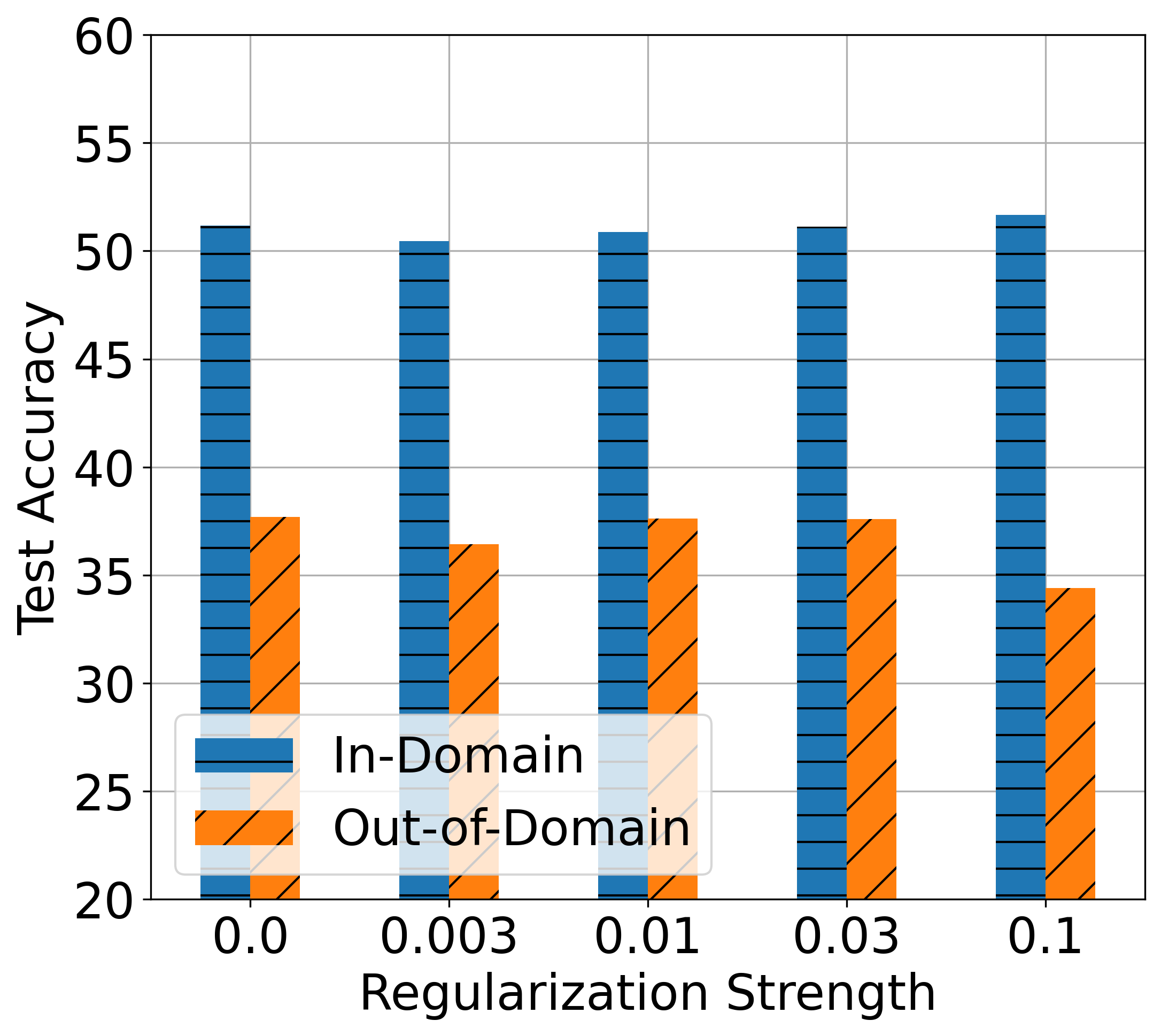}
    \caption{\textbf{Accuracy vs.\ regularization strength} using CNN and CIFAR10, with Dir(0.1), $\N(20)$.}
    \label{fig:ablation-regularization}
\end{minipage}
\hfill
\begin{minipage}[t]{0.48\columnwidth}
    \centering
    \includegraphics[width=\linewidth]{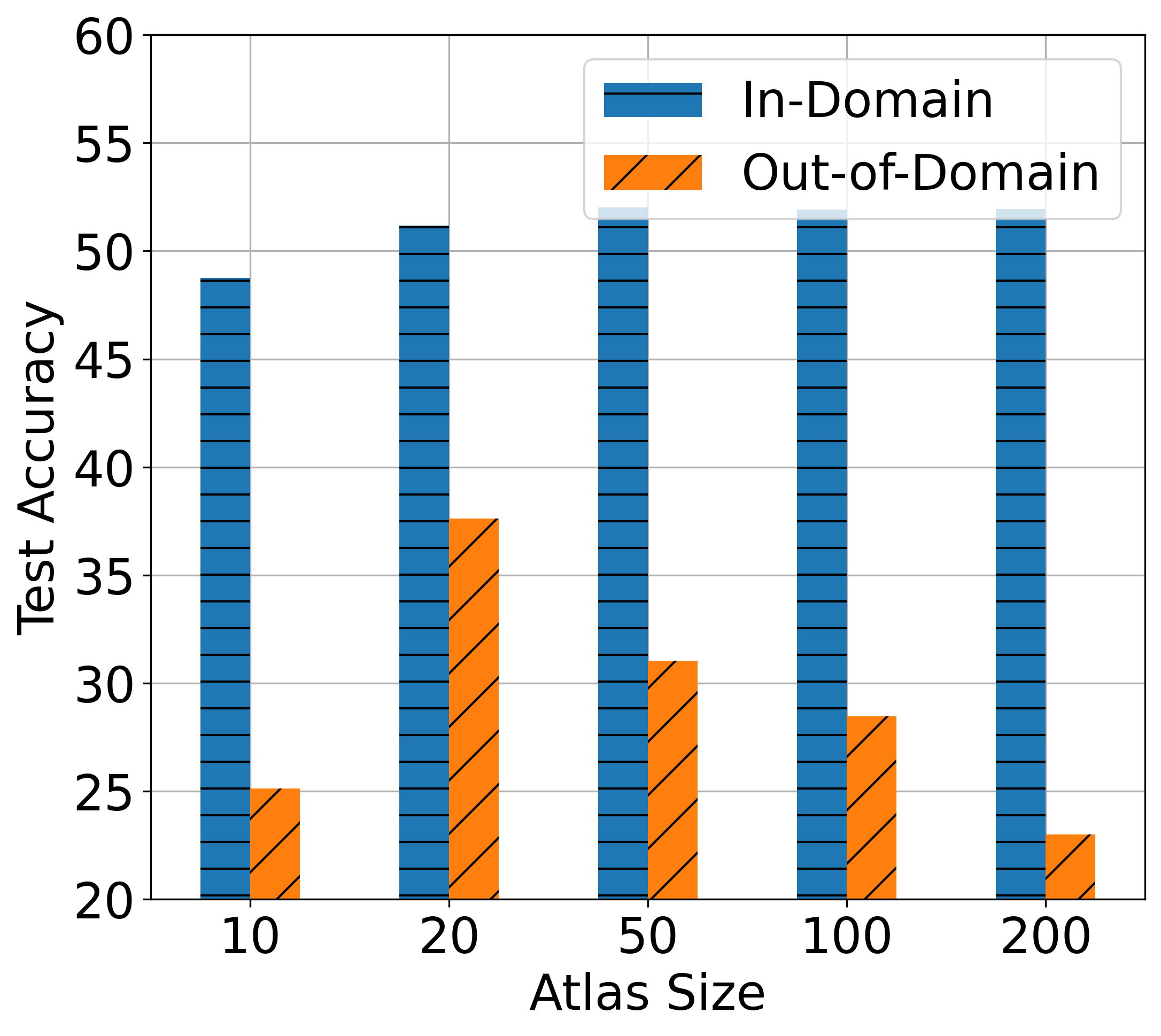}
    \caption{\textbf{Accuracy vs.\ atlas size} using CNN and CIFAR10, with Dir(0.1), $\N(20)$. \textit{10 clients are sampled per round}.}
    \label{fig:ablation-atlas-size}
\end{minipage}
\vspace{-.5em}
\end{figure}
\section{Future Work}
\label{sec:fw}
In this paper, we introduce the concept of hybrid data regimes and propose a federated dual learning framework, \texttt{Feddle}, to harness the strengths of both server-side and decentralized data. As a fundamental framework, \texttt{Feddle} opens many directions for future study and improvement. For instance, incorporating adaptive model updates, as explored in recent FL research~\citep{wang2023tackling,wang2024fadas}, could strengthen the anchors that comprise the model atlas. Additionally, the fallback mechanism can be refined for greater effectiveness, such as adaptively selecting strategies based on the delay status.
Another promising direction involves leveraging unsupervised objectives to exploit unlabeled server data, thereby enriching the available data resources. Furthermore, optimizing the computational efficiency of \texttt{Feddle} through techniques such as quantization or sparsification could improve its applicability to extremely large-scale settings. For example, the anchors can be quantized or sparsified before performing the coefficient search (c.f.\  \cref{eq:mmnn,eq:poop}).
Finally, integrating \texttt{Feddle} with differential privacy or trusted execution environments is an exciting area for addressing data privacy concerns. We plan to explore these avenues for improvement in future work.

\end{document}